\documentclass{cmslatex}
\usepackage[paperwidth=7in, paperheight=10in, margin=.875in]{geometry}
\usepackage{bm}
\usepackage[backref,colorlinks,linkcolor=red,anchorcolor=green,citecolor=blue]{hyperref}
\usepackage{amsfonts,amssymb}
\usepackage{amsmath}
\usepackage{graphicx}
\usepackage{cite}
\usepackage{enumerate}
\usepackage{amsfonts}
\usepackage{graphicx}
\usepackage{epstopdf}
\usepackage{algorithmic}
\renewcommand{\theequation}{\arabic{section}.\arabic{equation}}
\newcommand{\email}[1]{\texttt{#1}}

\allowdisplaybreaks

\begin{document}
\title{Some Super-approximation Rates of ReLU Neural Networks for Korobov Functions\thanks{Received date, and accepted date (The correct dates will be entered by the editor).}}

%For each author, make a block with the following macros:

\author{Yuwen Li\thanks{School of Mathematical Sciences, Zhejiang University, Yuhangtang Road 866,
Hangzhou 310058, China (\email{liyuwen@zju.edu.cn}).}
\and Guozhi Zhang\thanks{School of Mathematical Sciences, Zhejiang University, Yuhangtang Road 866,
Hangzhou 310058, China (\email{gzzh@zju.edu.cn}).}}

\pagestyle{myheadings} \markboth{Super-approximation Rates of ReLU Neural Networks}{Yuwen Li and Guozhi Zhang} 
\maketitle
\begin{abstract}
This paper examines the $L_p$ and $W^1_p$ norm approximation errors of ReLU neural networks for Korobov functions. In terms of network width and depth, we derive nearly optimal super-approximation error bounds of order $2m$ in the $L_p$ norm and order $2m-2$ in the $W^1_p$ norm, for target functions with $L_p$ mixed derivative of order $m$ in each direction. The analysis leverages sparse grid finite elements and the bit extraction technique. 
Our results improve upon classical lowest order $L_\infty$ and $H^1$ norm error bounds and demonstrate that the expressivity of neural networks is largely unaffected by the curse of dimensionality.
\end{abstract}


\section{Introduction}
In recent years, Deep Neural Networks (DNNs) have emerged as a crucial technology in the fields of machine learning and artificial intelligence and achieved tremendous success.  The remarkable success of DNNs has inspired numerous practical applications in computer science and engineering, including image recognition, natural language processing, autonomous driving, scientific computing, etc. Towards an explainable framework of machine learning,
understanding the approximation power of DNNs is a fundamental question and remains an active research area. In particular, one challenge in this area is characterizing the approximation error bounds for a certain function class in terms of network depth and width, see, e.g., \cite{yarotsky18a,Yarotsky_Zhevnerchuk2020,Shen_Yang_Zhang2019,lu_shen_yang_zhang2021deep,ShenYangZhang20b,shen2021deepfloor,DeVoreHaninPetrova2021,Zhang_shen_yang2022,He_Li_Xu2022ReLuDNN,Zhang_Lu_Zhao2024,daubechies2022nonlinear,SiegelXuFoCM,Jonathan2023,yang_wu_yang_xiang2023nearly,Yang2024Sobolev,he2024optimal,LiSiegel2024}. 

Recently, the bit extraction technique has gained renewed attention for verifying if certain functions exhibit super-approximation rates using well-designed network architectures.  This technique was invented by \cite{Bartllet1998,anthony1999neural} to compute VC-dimension bounds of DNNs, as discussed in \cite{Blumer_Ehrenfeucht_Haussler_Warmuth1989,Harvey_Liaw_Mehrabian2017,Bartlett_Harvey_Liaw_Mehrabian2019}. In  \cite{yarotsky18a,ShenYangZhang20b,lu_shen_yang_zhang2021deep,Jonathan2023,Liu_Liang_Chen2024,Yang2024Sobolev}, the bit extraction technique was further used to produce nearly optimal super-approximation error bounds of DNNs for continuous functions and Sobolev functions, while the order of approximation rate deteriorates as the dimension grows.

To partially cure the curse of dimensionality, Korobov regularity has been extensively applied in approximation error analysis of neural networks, with a focus on shallow networks (\cite{Blanchard_Bennouna2022,Liu_Mao_Zhou2024}), fully connected deep neural networks (\cite{Hadrien_Du2019,Yang_Lu2024}), convolutional deep neural networks (\cite{Mao_Zhou2022,lizhang2025korobov,fang2025_2dCNN}), and quantum circuits (\cite{Aftab_yang2024approximating}). Analysis in the aforementioned works relies on interpolation error estimates on sparse grids  (\cite{Bungartz_Dornseifer1998,Bungartz_Griebel2004}). We shall follow this line and derive  novel near-optimal approximation error bounds of arbitrary order for Korobov functions.

\subsection{Main Results}
Prior to this work, \cite{Yang_Lu2024} has derived approximation error bounds for Korobov functions (see Definition \ref{def:Korobov}) in $X^2_{\infty}(\Omega)$ using DNNs with a Rectified Linear Unit (ReLU) activation $\sigma(x):=\max(x,0)$, measured in the $L_{p}(\Omega)$ norm and the $W^1_2(\Omega)$ Sobolev norm. 

\begin{definition}[Korobov space]\label{def:Korobov}
Given $1\leq p\leq\infty$ and an integer $m\geq 2$, the Korobov space $X^m_p(\Omega)$ on the hypercube $\Omega=[0,1]^d$ is defined as 
\begin{align*}
X^m_p(\Omega):=&\big\{f\in L_p(\Omega): D^{\bm\alpha}f \in L_p(\Omega)\text{ for each multi-index}\\
&\bm{\alpha}\in\mathbb{N}^d\text{ with }|\bm \alpha|_{\infty} \leq m,~f|_{\partial \Omega}=0\big\}.
\end{align*}
The semi-norm of the space is given by $$|f|_{m,p}:=\left\|\frac{\partial^{md}f}{\partial x_1^m\cdots\partial x_d^m}\right\|_{L_p(\Omega)}.$$
\end{definition}
\vskip2mm
A natural question is whether increasing the smoothness of the target function leads to improved approximation rates.
In Theorem \ref{Th:mainresult}, we establish higher order super-approximation $L_p$ error bounds of ReLU DNNs for Korobov functions in $X^m_p(\Omega)$. 
\vskip2mm
\begin{theorem}\label{Th:mainresult}
Given any $f \in X^m_p(\Omega)$ with $1\leq p<\infty$ and $m\geq 2$, for any $W,L \in \mathbb{N}_{+}$, there is a function $\phi$ implemented by a ReLU DNN with width $W$ and depth
$L$ such that
$$
\|f-\phi\|_{L_p(\Omega)} \leq \widetilde{C}_1|f|_{m,p}  W^{-2m}L^{-2m}(\log( 8W))^{2(m+2)(d+1)}(\log( 8L))^{2(m+2)(d+1)},
$$
where $\widetilde{C}_1=(C_1C_2)^4C_3$, $C_1=112d(2d)^d$, $C_2=320d^2$, and $C_3 = (2d)^{3d}$ if $m=2$; $\widetilde{C}_1$ is a constant depending only on $m$ and $d$  if $m\geq3$.  
\end{theorem}

\textbf{Sketch of proof}: For any $f \in X^m_p(\Omega)$, the proof proceeds in three main steps. First, we utilize the sparse grid interpolation $\Pi_n^mf(\bm{x}) = \sum_{|\bm{l}|_1 \leq n+d-1} \sum_{\bm{i} \in \mathcal{I}_{\bm{l}}} v_{\bm{l}, \bm{i}}^m \phi_{\bm{l}, \bm{i}}^m(\bm{x})$, which is known to approximate $f$ with the desired error bounds. Second, for each level $\bm{l}$, we construct a ReLU DNN $\phi_{\bm{l}}$ to approximate the sub-sum $\sum_{\bm{i} \in \mathcal{I}_{\bm{l}}} v_{\bm{l}, \bm{i}}^m \phi_{\bm{l}, \bm{i}}^m$ within a subdomain $\Omega_{\varepsilon} \subset \Omega$. Specifically, since the basis functions $\phi_{\bm{l}, \bm{i}}^m$ have disjoint supports for a fixed level $\bm{l}$, the sub-sum simplifies locally to a single term $v_{\bm{l}, \bm{i}}^m \phi_{\bm{l}, \bm{i}}^m$. Each such term is approximated by a network that integrates a product sub-network, a sub-network representing the basis function, and one approximating the coefficients. Finally, we aggregate these networks $\phi_{\bm{l}}$ for all $|\bm{l}|_1 \leq n+d-1$ to form the ultimate neural network. For $p < \infty$, the approximation error on $\Omega_{\varepsilon}$ can be extended to the entire domain $\Omega$ by choosing a sufficiently small $\varepsilon$, thereby preserving the desired convergence rates.

The authors of \cite{Yang_Lu2024} conjectured that the super-approximation error bound of ReLU DNN for target functions in $X^2_p(\Omega)$ is $O(W^{-4+\frac{1}{p}}L^{-4+\frac{1}{p}})$ (up to logarithmic factors). The lowest order version of Theorem \ref{Th:mainresult} disproves this conjecture and illustrates that the ReLU DNN achieves the nearly optimal approximation error bound regardless of the integrability index $p$.

When solving PDEs using DNNs, the error analysis requires approximation error bounds measured in the $W^1_p$ Sobolev norm
$$
\|f\|_{W^1_p(\Omega)}:=\left(\sum_{0\leq |\bm \alpha|\leq 1}\|D^{\bm \alpha}f\|_{L_p(\Omega)}^p\right)^{\frac{1}{p}}.
$$
For this reason, we also derive the following result.
\vskip2mm
\begin{theorem}\label{Th:mainresultW1p}
Given any $f \in X^m_p(\Omega)$ with $1\leq p<\infty$ and $m\geq 2$, for any $W,L \in \mathbb{N}_{+}$, there is a function $\phi$ implemented by a ReLU DNN with width $W$ and depth
$L$ such that
$$
\|f-\phi\|_{W^1_p(\Omega)} \leq \widetilde{C}_2 |f|_{m,p} W^{-2(m-1)}L^{-2(m-1)}(\log( 8W))^{2(m+1)(d+1)}(\log( 8L))^{2(m+1)(d+1)},
$$
where $\widetilde{C}_2$ is a constant depending only on $m,p$ and $d$.     
\end{theorem}

{ \textbf{Sketch of proof}: For any $f \in X^m_p(\Omega)$, we adopt a distinct methodology for the $W^1_p$ norm analysis. This approach bypasses the need for the domain-extension process (from $\Omega_{\varepsilon}$ to $\Omega$) typically required in the $L_p$ case for $p<\infty$. First, we utilize the sparse grid interpolation $\Pi_n^mf(\bm{x}) = \sum_{|\bm{l}|_1 \leq n+d-1} \sum_{\bm{i} \in \mathcal{I}_{\bm{l}}} v_{\bm{l}, \bm{i}}^m \phi_{\bm{l}, \bm{i}}^m(\bm{x})$, which is known to approximate $f$ with the desired error bounds in the $W^1_p$ norm. Second, we use a partition of unity $\{g_{\bm{k}}\}$ satisfying $\sum_{\bm{k}} g_{\bm{k}} \equiv 1$ on $\Omega = \cup \Omega_{\bm{k}}$, where each $g_{\bm{k}}$ is efficiently approximated on $\Omega_{\bm{k}}$ by a ReLU DNN $\phi_{\bm{k}}$. Within each subdomain $\Omega_{\bm{k}}$, we construct a ReLU DNN $\psi_{\bm{k}}^{\bm{l}}$ to approximate the sub-sum $\sum_{\bm{i} \in \mathcal{I}_{\bm{l}}} v_{\bm{l}, \bm{i}}^m \phi_{\bm{l}, \bm{i}}^m$ in both $W^1_p$ and $L_p$ norms, following a procedure analogous to the $L_p$ case. Next, we aggregate these sub-networks $\psi_{\bm{k}}^{\bm{l}}$ over all levels $|\bm{l}|_1 \leq n+d-1$ to form a localized network $\psi_{\bm{k}}$ that approximates $f$ on $\Omega_{\bm{k}}$ by incorporating the interpolation error. Finally, the global approximation error on $\Omega$ is decomposed across the subdomains $\Omega_{\bm{k}}$ via the inequality $\|f-\phi\|\leq \sum_{\bm k}\|fg_{\bm k}-\widetilde{\phi}(\phi_{\bm k},\psi_{\bm k})\|$, where $\widetilde{\phi}$ denotes the product network derived from Lemma \ref{Le:product_W1p_error}.
}

The approximation errors measured in the $L_p$ and $W^1_p$ norms demonstrate that ReLU DNNs can effectively double the approximation rate of conventional methods—such as continuous function approximators (cf. \cite{Yang_Lu2024,DeVoreHaninPetrova2021,Jonathan2023,Yang2024Sobolev})—with respect to the network depth $L$ and width $N$. This phenomenon is thus referred to as the \textit{super-approximation} of ReLU DNNs. We can further establish that the error bounds of ReLU DNNs in Theorems \ref{Th:mainresult} and \ref{Th:mainresultW1p} are nearly optimal: For any $\delta,C>0$, there exist $f\in X^m_p(\Omega)$ with $|f|_{m,p}\leq1$ and $W,L\in \mathbb N_+$ such that
\begin{align*}
\inf_{g\in\mathcal{NN}(W,L)}\|f-g\|_{L_p(\Omega)}&\geq CW^{-2m-\delta}L^{-2m-\delta},\\
\inf_{g\in\mathcal{NN}(W,L)}\|f-g\|_{W^1_p(\Omega)}&\geq CW^{-2(m-1)-\delta}L^{-2(m-1)-\delta}.
\end{align*}
This result can be proved using \cite[Theorem 22]{Jonathan2023} and the VC-dimension bound of ReLU DNNs (\cite{Bartlett_Harvey_Liaw_Mehrabian2019,YangYangXiang2023}). 
\vskip2mm
\begin{remark}
By the definition of Korobov spaces, we observe the inclusion $W^{md}_p(\Omega) \cap \{f: f|_{\partial\Omega}=0\} \subset X^m_p(\Omega)$. This inclusion is strict, i.e., there exist functions $f \in X^2_p(\Omega)$ such that $f \notin W^{2d}_p(\Omega)$ for any $1 \leq p \leq \infty$ (cf.~\cite{Mao_Zhou2022}).  Consequently, existing $L_p$ norm error bound $O((WL)^{-2s/d})$ of DNN approximation for $W^s_p(\Omega)$ functions  (cf.~\cite{Jonathan2023,Yang2024}) with $s=md$ cannot be directly applied to Korobov functions to achieve the error bound $O((WL)^{-2m})$ (ignoring logarithmic factors). The mixed regularity of Korobov functions necessitates a non-trivial DNN approximation theory, which mitigates the curse of dimensionality by leveraging sparse grid approximation.
\end{remark}
{ \subsection{Other Network Architectures}
As indicated in Corollary 2 in \cite{YangYangXiang2023}, for the approximation of Sobolev functions $f\in W^m_{\infty}(\Omega)$ in the $W^n_{\infty}$  norm with $n\geq 2$ using ReLU-$\text{ReLU}^2$ DNN (where the activation function applied to each neuron can be chosen as ReLU or its square) with width $O(W\log W)$ and depth $O(L\log L)$, the superconvergence rates $O((WL)^{-2(m-n)/d})$ can be achieved  by employing a smoother unit decomposition similar to $g_{\bm m}$, i.e., utilizing piecewise polynomials of degree $n$. This methodology can actually be used to obtain the approximation of higher order Korobov functions using ReLU-$\text{ReLU}^2$ DNNs in the norm of $W^n_p$  with $n\geq 2$, with the help of the error estimates of sparse grid interpolation  in this norm. This error  can actually be computed using the inverse estimates akin to the process of calculating the $W^1_p$ error.}

{ Recently, Floor-ReLU networks (i.e., the activation function applied to each neuron can be chosen as ReLU or Floor function) with super approximation power have been introduced in \cite{shen2021deepfloor} to efficiently approximate H\"older functions on $\Omega$. By observing that the proof process in \cite[Proposition 3.2]{shen2021deepfloor} have provided that for $K=W^L$ and $\xi_{i}\in[0,1]$ for $i=1,2,\ldots,K$, there exists a function $\psi\in \mathcal{NN}(2W^2+5W,L(7L-1)+2)$
such that
$
|\psi( i)-\xi_{i}|\leq 2^{-WL}
$,
we can provide the results for approximating Korobov functions on $\Omega$ by the Floor-ReLU DNN. Similar to ReLU DNN approximation in the norm of $W^1_p$,  establishing a series of results about approximating polynomial by Floor-ReLU DNN in the $W^1_p$
norm will lead to the approximation results for higher order Korobov functions using Floor-ReLU DNNs, which will be left for future research.}

{ Recently, the authors of \cite{Liu_Liang_Chen2024} have shown that ResNet (\cite{he2016deepresnet,Lin_Jegelka2018ResNet,Oono_suzuki2019approximation,Li_Lin_Shen2022deepResnet}) with constant width and depth $O(L)$ can
approximate Lipschitz continuous function $f\in C([0,1]^d)$ with the optimal approximation rates $O(\omega_f(L^{-2/d}))$, where $\omega_f$ is the modulus of continuity of the function $f$. Additionally, the versions of ResNet similar to Lemma \ref{Le:stepfunction}, Lemma \ref{Le:fitsample}, and Lemma \ref{Le:polynomial_W1p} for $L_{\infty}$ norm are presented. By combining these results with the representation of basis functions by ReLU networks as demonstrated in the proof of Theorem \ref{Th:mainresult}, one can derive the approximation rates for Korobov spaces using ResNet through analogous methodologies. For brevity, we will not delve into the proof here. Moreover, we can obtain approximation results of ResNet for Korobov functions in the $W^1_p$ norm, resembling those outlined in Theorem \ref{Th:mainresultW1p}, by formulating the $W^1_p$ ResNet counterparts of Lemma \ref{Le:stepfunction}, \ref{Le:fitsample}, and \ref{Le:polynomial_W1p}.
}

The rest of this paper is organized as follows. In Section \ref{Sec:pre}, we provide preliminaries on sparse grids and several lemmas required  to prove our main results. In Section \ref{Sec:ReLU_DNNLp} and \ref{Sec:ReLU_DNNW1p}, we prove Theorem \ref{Th:mainresult} and \ref{Th:mainresultW1p}, respectively. 
Section \ref{Sec:conclusion} is devoted to concluding remarks.

%\vspace{2mm}
%\vspace{2mm}
%\vspace{2mm}
\section{Preliminaries}\label{Sec:pre}

\subsection{Notation}
\begin{enumerate}

\item [$\bullet$] Let $\mathbb N_+$ denote the set of all positive integers, i.e., $\mathbb N_+=\{1,2,\ldots\}$ and let $\mathbb N:=\mathbb N_+\cup \{0\}$.
%\item[$\bullet$] The floor activation function (Floor) is defined by  $\lfloor x\rfloor := \max \{ n : n \leq 
%x, n \in \mathbb{Z}\}$ for any $x\in \mathbb R$.
\item [$\bullet$]
For any $x \in \mathbb R$, let $\lceil x\rceil := \min \{ n : n > 
x, n \in \mathbb{Z}\}$.

\item [$\bullet$]
A DNN function $\phi: \mathbb R^d\longrightarrow \mathbb R$ can be realized through the following function compositions 
$$
\phi=\mathcal{A}_L\circ \sigma \circ\mathcal{A}_{L-1}\circ \sigma\circ\cdots\sigma\circ \mathcal{A}_1\circ \sigma\circ\mathcal{A}_0,
$$
where $\mathcal{A}_i(\bm x_i):=\bm W_i\bm x_i+\bm b_i$, with $\bm W_i\in \mathbb R^{N_{i+1}\times N_{i}}$ and $\bm b_i\in \mathbb R^{N_{i+1}}$ for all $\bm x_i\in \mathbb R^{N_{i}}$.
The activation function $\sigma$ acts on vectors in a component-wise way.
In addition, we set $N_0=d,N_{L+1}=1$. And $W:=\max_{1\leq i\leq L}N_i$ and $L$ refer to the width and depth of the DNN function $\phi$, respectively.
%\item [$\bullet$] In this paper, $W$ and $L$ refer to the width and depth of a neural network (up to logarithmic factors), respectively.
%\item [$\bullet$] Let the boldface symbol $\bm W$ and $\bm b$ represent a matrix and vector, respectively.
\item [$\bullet$] By $\mathcal{NN}(W,L)$ we denote the set of ReLU DNNs with width $W$ and depth $L$. By $\mathcal{NN}_{d,k}(W,L)$ we
emphasize that the input and output dimensions are $d$ and $k$, respectively.

\item [$\bullet$]
Let $\mathbf{1}_E$ be the characteristic function on a set $E$, i.e., 
$$
\mathbf{1}_{E}= \begin{cases}1, & x \in E, \\ 0, & x\notin E.\end{cases}
$$

\item [$\bullet$] In our analysis, $C_{m,p,d,\ldots}$ is a generic constant that may change from line to line and is dependent only on $m, p,d,\ldots$

\item [$\bullet$]
For any multi-index $\bm \alpha=(\alpha_1,\alpha_2,\ldots,\alpha_d)\in \mathbb N_+^d$ and $\bm \beta=(\beta_1,\beta_2,\ldots,\beta_d)\in \mathbb N_+^d$, by $\bm \alpha\leq \bm \beta$ we denote $\alpha_i \leq \beta_i$ for each $i=1,2,\ldots,d$.

\item [$\bullet$]
For any multi-index $\bm \alpha=(\alpha_1,\alpha_2,\ldots,\alpha_d)\in \mathbb N_+^d$, let 
$$
D^{\bm \alpha}:=\frac{\partial^{|\bm \alpha|}}{\partial x_1^{\alpha_1}\partial x_2^{\alpha_2}\cdots\partial x_d^{\alpha_d}}
$$
and $|\bm \alpha|:=\alpha_1+\alpha_2+\cdots+\alpha_d$.

\end{enumerate}

%For given $n\in \mathbb N_+$ and any $\bm l$ satisfing $|\bm l|_1\leq n+d-1$, it is natural to give a version of the trifling region  for sparse grid similar to that in \cite{shen2021deep}. The definition is
%precisely given in \cite{yang2024near}, we present it as follows for the convenience.
\subsection{Some Lemmas}
In this subsection, we present several lemmas required to prove our main results.
\vskip2mm
\begin{lemma}[Proposition 4.3 from \cite{lu_shen_yang_zhang2021deep}]\label{Le:stepfunction}
For any $W, L \in \mathbb{N}_{+}$ and $\varepsilon \in\left(0, \frac{1}{3 K}\right]$ with $K= W^2L^2 $, there exists a function $\phi \in \mathcal{NN}(4W+3,4L+5)$ such that
$$
\phi(x)=k \quad \text { if } x \in\left[\frac{k}{K}, \frac{k+1}{K}-\varepsilon \cdot \mathbf{1}_{\{k \leq K-2\}}\right]
$$   
for $k=0,1, \cdots, K-1$.
\end{lemma}

The following lemmas (cf.~\cite{Jonathan2023,JiaoWangYang2023}) are useful for the analysis of neural networks. 
\vskip2mm
\begin{lemma}\label{Le:concate}
Let $d,k,d_i,k_i,W_i,L_i\in \mathbb N_+$ and $\phi_i\in \mathcal{NN}_{d_i,k_i}(W_i,L_i)$ for all $i=1, \ldots, n$.

(1) For $W_1 \leq W_2, L_1 \leq L_2$, we have $\mathcal{N} \mathcal{N}_{d, k}\left(W_1, L_1\right) \subseteq \mathcal{N} \mathcal{N}_{d, k}\left(W_2, L_2\right)$.

(2) (Composition) If $k_1=d_2$, then $\phi_2 \circ \phi_1 \in \mathcal{N}_{d_1, k_2}\left(\max \left\{W_1, W_2\right\}, L_1+L_2\right)$.

(3) (Concatenation) Let $\phi(\bm x):=\binom{\phi_1(\bm x_1)}{\phi_2(\bm x_2)}$, then
$$
\phi \in \mathcal{NN}_{d_1+d_2,k_1+k_2}\left(W_1+W_2, \max\{L_1,L_2\}\right).
$$

(4) (Summation) 
If $d_i=d$ and $k_i=k$ for all $i=1, \ldots, n$, then 
\begin{align*}
\sum_{i=1}^n \phi_i&\in \mathcal{NN}_{d, k}\left(\sum_{i=1}^n W_i,\max _{1 \leq i \leq n} L_i\right),\\
\sum_{i=1}^n \phi_i&\in \mathcal{N} \mathcal{N}_{d, k}\left(\max _{1 \leq i \leq n} W_i+2 d+2 k, \sum_{i=1}^n L_i\right).    
\end{align*}
\end{lemma}
\vskip2mm
\begin{lemma}[Proposition 4.4 from \cite{lu_shen_yang_zhang2021deep}]\label{Le:fitsample}
Let $W, L, s \in \mathbb{N}_{+}$ and $0\leq \xi_i \leq 1$ for $i=0,1,\ldots, W^2L^2-1$, there exists a neural network $$\phi\in \mathcal{NN}(16 s(W+1) \log _2(8 W),5(L+2) \log _2(4 L))$$ 
such that $\phi(x)\in[0,1]$ for any $x \in \mathbb{R}$ and
$$\max_{0\leq i\leq W^2 L^2-1}\left|\phi(i)-\xi_i\right| \leq W^{-2 s} L^{-2 s}.$$   
\end{lemma}
\vskip2mm
The next lemma shows that $O(N_1N_2)$ sub-networks with width $O(W)$ and depth $O(L)$ sum up to a network of width $O(N_1W)$ and depth $O(N_2L)$.
%The proof can essentially be derived from Lemma \ref{Le:concate}. For clarity, we state the following result and provide a brief proof for completeness.
\vskip2mm
\begin{lemma}\label{Le:size_transition}
Let $W,L,N_1,N_2, d\in \mathbb N_+$. Suppose that $\{\phi_i\}_{i=1}^{N_1N_2}$ is a collection of neural networks with $$\phi_i\in \mathcal{NN}_{d,1}(W,L)\quad  \text{for any}\,\,i=1,2,\ldots,N_1N_2. $$
Then we have 
$$
\sum_{i=1}^{N_1N_2}\phi_i \in\mathcal{NN}_{d,1}(N_1W+2d+2,N_2L).
$$  
\end{lemma}
\begin{proof}
%We use the induction method for depth $N$ to prove the required results.  
%First, we consider the case $L=1$. If $W,N \leq 2$, the result is immediately valid. For $W,N\geq 2$, by noticing that each network function is in a set consisting of the network implemented by a two-layer ReLU network with width $W$ in the first layer and $WN$ in the second layer, using Lemma 3.4 in \cite{shen2021deep} implies the relation $\mathcal{NN}_d(WN,L)\subset \mathcal{NN}_d(2W+2,(N+1))$ holds. Next, suppose the results hold for $L-1$. We would construct a network function to ensure the results hold for $L$.
If we rewrite the summation as
$$
\phi=\sum_{i=1}^{N_1N_2}\phi_i=\sum_{i=1}^{N_2}\sum_{j=1}^{N_1}\phi_{i,j},
$$
where each $\phi_{i,j}$ is still of width $W$ and depth $L$, then for each $i=1,2\ldots, N_2$, we have $\sum_{j=1}^{N_1}\phi_{i,j}\in \mathcal{NN}_{d,1}(N_1W,L)$ based on the first result in the fourth part of Lemma \ref{Le:concate}.
By using the second result in the fourth part of Lemma \ref{Le:concate} with $n=N_2$ and $k=1$ there, we further have $\phi \in \mathcal{NN}_{d,1}(N_1W+2d+2,N_2L)$.
\end{proof}
%\begin{remark}\label{Re:size_transition_floor}
%The claim holds for Floor-ReLU network because the network reconstruction is independent of the choice of activation functions, and the $d$-dimensional identity map can be implemented by a ReLU network of width $2d$ and depth $1$.  \end{remark}
\vskip2mm
\begin{lemma}[Proposition 4 from \cite{YangYangXiang2023}]\label{Le:product_W1p_error}
For any $W, L \in \mathbb{N}_{+}$ and $a\geq2$, there is a  function $ \phi\in \mathcal{NN}(15 W,2L)$ such that $\|\phi\|_{W^1_{ \infty}\left((-a, a)^2\right)} \leq 12 a^2$ and
$$
\|\phi(x, y)-x y\|_{W^1_{ \infty}\left((-a, a)^2\right)} \leq 6 a^2 W^{-L}
$$
Furthermore,
$
\phi(0, y)=\frac{\partial \phi(0, y)}{\partial y}=0
$ for any $ y \in(-a, a)$. 
\end{lemma}
\vskip2mm
Next we construct a neural network that efficiently approximates the product of finite input variables in the $W^1_{\infty}$ norm over a larger bounded domain. Furthermore, we give the upper bound of the $W^1_{\infty}$ norm of the resulting network function. The proof of Lemma~\ref{Le:polynomial_W1p} is similar to \cite[Lemma 3.5]{HonYang2022simultaneous} with the help of Lemma \ref{Le:product_W1p_error}.
When the intervals $[-2,2]$ and $[-c,c]$ below are replaced by $[0,1]$, the result described corresponds to \cite[Proposition 5]{YangYangXiang2023}.
\vskip2mm
\begin{lemma}\label{Le:polynomial_W1p}
Given any $d \in \mathbb{N}_{+}$ with $d \geq 2$ and $c > 2$, let $\Omega_c=[-2,2]^{d}\times[-c,c]$ and $a=\max\{2^{d+1},c\}$. Then for any $W, L \in \mathbb{N}_{+}$, there exists a function $\phi\in \mathcal{NN}(15(W+1)+2d-1,14 d^2L)$ such that $\|\phi\|_{W^1_{ \infty}\left(\Omega_c\right)} \leq 12a^{2}$ and
$$
\left\|\phi(\bm{x})-x_1 x_2 \cdots x_{d+1}\right\|_{W^1_{ \infty}\left(\Omega_c\right)} \leq  14 a^{4} (W+1)^{-7d L}.
$$    
\end{lemma}

%\subsection{Proof of Lemma \ref{Le:polynomial_W1p}}\label{Sec:polynomial_W1p_proof}
\begin{proof}
By Lemma \ref{Le:product_W1p_error}, there exists a function $\phi_1\in \mathcal{NN}(15(W+1),14dL)$ such that $\|\phi_1\|_{W^1_{\infty}([-2^{d+1},2^{d+1}]^2)}\leq 12\cdot 2^{2d+2}$ and
\begin{equation}\label{Eq:phi1W1pbound}
\|\phi_1(x,y)-xy\|_{W^1_{\infty}([-2^{d+1},2^{d+1}]^2)}\leq 6\cdot 2^{2d+2}\cdot(W+1)^{-7dL}.
\end{equation}
If $d=2$, it holds that $\phi_1,\frac{\partial \phi_1}{\partial x},\frac{\partial \phi_1}{\partial y}\in [-2^{3},2^{3}]=[-2^{d+1},2^{d+1}]$ for any $x,y\in [-2,2]$.  If $d\geq 3$,  for $i\leq d-2$,
suppose that there exsits a function $\phi_i\in \mathcal{NN}(15(W+1)+2i-1,14idL)$ such that
\begin{equation}\label{Eq:phi_i_W1pbound}
\Big\|\phi_i(x_1,x_2,\ldots,x_{i+1})-\prod_{j=1}^{i+1}x_j\Big\|_{W^1_{\infty}([-2,2]^{i+1})}
\leq (2^i-1)\cdot6\cdot 2^{2d+2}\cdot(W+1)^{-7dL},
\end{equation}
and $\phi_i,\frac{\partial \phi_i}{\partial x_j}\in[-2^{d+1},2^{d+1}]$ if $x_j\in [-2,2]$ for $j=1,2,\ldots,i+1$.
Define $$\phi_{i+1}(x_1,x_2,\ldots,x_{i+2}):=\phi_1(\phi_i(x_1,x_2,\ldots,x_{i+1}),x_{i+2}).$$
By Lemma \ref{Le:concate},  $\phi_{i+1}\in\mathcal{NN}(15(W+1)+2i+1,14(i+1)dL)$. Using the error bounds \eqref{Eq:phi1W1pbound} and \eqref{Eq:phi_i_W1pbound} implies that
$$
\Big\|\phi_{i+1}(x_1,x_2,\ldots,x_{i+2})-\prod_{j=1}^{i+2}x_j\Big\|_{W^1_{\infty}([-2,2]^{i+2})}
\leq (2^{i+1}-1)\cdot 6\cdot 2^{2d+2}(W+1)^{-7dL}.
$$
In addition, $\phi_{i+1},\frac{\partial \phi_{i+1}}{\partial x_j}\in[-2^{d+1},2^{d+1}]$ if $x_j\in [-2,2]$ for $j=1,2,\ldots,i+2$.
Let $\phi:=\phi_{d-1}$, by the principle of induction, 
$$
\phi\in\mathcal{NN}(15(W+1)+2d-3,14d(d-1)L),
$$ 
and for $\bm{x}=\left(x_1,x_2,\ldots,x_d\right)^{\top} \in[-2,2]^d $,
\begin{equation}\label{Eq:phi_d_W1p}
\|\phi(\bm{x})-x_1x_2\ldots x_d\|_{W^1_{\infty}([-2,2]^d)} \leq 2^{3d+6} (W+1)^{-7dL}.
\end{equation}
Using Lemma \ref{Le:product_W1p_error} again with $a=\max\{2^{d+1},c\}$ , there is a  function $ \psi\in \mathcal{NN}(15 (W+1),14dL)$ such that $\|\psi\|_{W^1_{ \infty}\left([-a, a]^2\right)} \leq 12 a^2$ and
\begin{equation}\label{Eq:psi_2_W1p}
\|\psi(x, y)-x y\|_{W^1_{ \infty}\left([-a, a]^2\right)} \leq 6 a^2 (W+1)^{-7dL}
\end{equation}
Let $\Phi(\bm x):=\psi(\phi(\bm x^{\prime}),x_{d+1})$ for any $\bm x=(\bm x^{\prime},x_{d+1})\in \mathbb R^{d+1}$. By  Lemma \ref{Le:concate}, \eqref{Eq:phi_d_W1p}, and \eqref{Eq:psi_2_W1p}, we have $\Phi\in \mathcal{NN}(15(W+1)+2d-1,14 d^2L)$,
$\|\Phi\|_{W^1_{ \infty}\left(\Omega_c\right)} \leq 12a^{2}$ and
$$
\left\|\Phi(\bm{x})-x_1 x_2 \cdots x_{d+1}\right\|_{W^1_{ \infty}\left(\Omega_c\right)} \leq 14a^4\cdot (W+1)^{-7d L}.
$$ 
The proof is complete.
\end{proof}  
\vskip2mm
\begin{lemma}\label{Le:prod_approx_difference}
Let $m\in \mathbb N_+$ and
$a_i,b_i\in \mathbb R$ with $|a_i|\leq 2$ and $|b_i|\leq 1$ for any $i=1,2,\ldots,m$. Suppose that for some $\varepsilon>0$ small enough 
$
|a_i-b_i|\leq \varepsilon
$ holds
for any $i=1,2\ldots,m$. Then we have
$$
\Big|\prod_{i=1}^ma_i-\prod_{i=1}^mb_i\Big|\leq 2^m \varepsilon.
$$
\end{lemma}
\begin{proof}
We proceed by induction. For $m=1$, it is valid from the assumption. Assume the inequality holds for $m=k-1$, i.e., 
$$
\Big|\prod_{i=1}^{k-1}a_i-\prod_{i=1}^{k-1}b_i\Big|\leq 2^{k-1} \varepsilon.
$$
Now we consider the case $m=k$. By the hypothesis of induction and the condition $|a_i|\leq 2$ and $|b_i|\leq 1$ for any $i=1,2,\ldots,d$, we have
$$
\begin{aligned}
\Big|\prod_{i=1}^{k}a_i-\prod_{i=1}^kb_i\Big|&\leq |a_k-b_k|\cdot\Big|\prod_{i=1}^{k-1}a_i\Big|+|b_k|\cdot\Big|\prod_{i=1}^{k-1}a_i-\prod_{i=1}^{k}b_i\Big|\\
&\leq 2^{k-1}\varepsilon+2^{k-1}\varepsilon\\
&\leq 2^k \varepsilon.
\end{aligned}
$$
By the principle of induction, the inequality holds for any $m\in \mathbb N_+$, thus we finish the proof.
\end{proof}

\subsection{Interpolation on Sparse Grids}
For any $\bm l\in \mathbb N_+^d$, we consider a $d$-dimensional rectangular grid $\mathcal{T}_{\bm l}$ in $\Omega=[0,1]^d$ with mesh size $\bm{h}_{\bm{l}}=\left(h_{l_1}, \ldots, h_{l_d}\right)$, i.e., the mesh size in the $i$-th coordinate direction is  $h_{l_i}=2^{-l_i}$. For each $\bm{0}\leq \bm{i}\leq 2^{\bm{l}}=(2^{l_1}
,\ldots, 2^{l_d})$, the $\bm{i}$-th grid point in $\mathcal{T}_{\bm l}$ is 
\begin{equation*}
\bm{x}_{\bm{l}, \bm{i}}=\left(x_{l_1, i_1}, \ldots, x_{l_d, i_d}\right):=(i_1h_{l_1},\ldots,i_dh_{l_d}).
\end{equation*}
Let $
\phi(x):= \sigma(1-|x|)
$ and $\phi_{l_j, i_j}(x_j):=\phi\left(\frac{x_j-i_j \cdot h_{l_j}}{h_{l_j}}\right)$ for any $j=1,2,\ldots,d$. We define the tensor product basis 
\begin{align*}
\phi_{\bm{l}, \bm{i}}(\bm{x}):=\prod_{j=1}^d \phi_{l_j, i_j}(x_j).
\end{align*} 
For any $f\in X^2_p(\Omega)$ with $|f|_{2,p}\leq1$ and $n\in \mathbb N_+$, $1\leq p<\infty$, consider the lowest-order sparse grid interpolation (\cite{Bungartz_Griebel2004})
\begin{equation}\label{Eq:sparse_interpolation_2}
\Pi_nf(\bm x)=\sum_{|\bm{l}|_1\leq n+d-1} \sum_{\bm{i} \in\mathcal{I}_{\bm l}} v_{\bm{l}, \bm{i}} \phi_{\bm{l}, \bm{i}}(\bm x),
\end{equation}
where $\mathcal{I}_{\bm l}$ is a hierarchical index set given by 
\begin{equation}\label{Eq:i_l}
\mathcal{I}_{\bm l}:=\left\{\bm{i} \in \mathbb{N}_+^d: \bm{1} \leq \bm{i} \leq 2^{\bm{l}}-\mathbf{1},~i_j \text { is odd for all } 1 \leq j \leq d\right\},
\end{equation}
and the coefficient $v_{\bm{l},\bm{i}}$ satisfying
$
|v_{\bm{l},\bm{i}}|\leq 1
$. It is proved in \cite{Mao_Zhou2022} that for a given $n$ satisfying $$\log \left(W{\left(\log  W\right)^{-d+1}}\right) \leq n \leq \log  W, $$ the approximation error for the interpolation function 
$
\Pi_nf$
approximating Korobov functions in $X^2_p(\Omega)$ in the $L_p$ norm is explicitly 
$$W^{-(2-\frac{1}{p})}(\log W)^{(3-\frac{1}{p})(d-1)}$$ for $1\leq p<\infty$. In fact, by using the support set of the basis function $\phi_{\bm l,\bm i}(\bm x)$ in the calculations for the bound of the coefficient $v_{\bm l, \bm i}$ there, we could get the following error bound
\begin{equation}\label{Eq:sparse_error}
\|f-\Pi_nf\|_{L_p(\Omega)}\leq W^{-2}(\log W)^{3(d-1)}\leq 2^{-2n}(nd)^{3(d-1)}.
\end{equation}
Let $m\geq 3$, $1\leq p<\infty$, and $n\in \mathbb N_+$, for any $f\in X^m_p(\Omega)$ with $|f|_{m,p}\leq 1$, consider the $(m-1)$-th order sparse grid
interpolation \begin{equation}\label{Eq:sparseinterpolation_m}
\Pi_n^mf(\bm{x})=\sum_{|\bm l|_1 \leq n+d-1} \sum_{\bm i \in \mathcal{I}_{\bm l}} v_{\bm l, \bm i}^m\phi_{\bm l, \bm i}^m(\bm {x}),
\end{equation}
where $\phi_{\bm l, \bm i}^m$ is the higher order sparse grid basis function of degree $\leq m-1$ in each direction (see \cite{Bungartz_Griebel2004} and \cite{lizhang2025korobov}), and the coefficient $v_{\bm{l},\bm{i}}^m$ satisfies
$
|v_{\bm{l},\bm{i}}^m|\leq C_{m,d}
$ for some fixed constant $C_{m,d}$.
For $n\in \mathbb N_+$, the sparse grid interpolation $\Pi_n^mf$ satisfies
\begin{equation}\label{Eq:Sparse_error_If_n}
\|f-\Pi_n^mf\|_{L_p(\Omega)} \leq C_{m,d}2^{-mn}n^{d-1},
\end{equation}
where $C_{m,d}$ is a constant depending only on $m$ and $d$.

\section{Approximation by ReLU DNN in \texorpdfstring{$L_{p}$}{Lp} Norm}\label{Sec:ReLU_DNNLp}
This section is devoted to the proof of  Theorem \ref{Th:mainresult}. First we define a region for sparse grids by excluding a small area   similar to that in \cite{lu_shen_yang_zhang2021deep}, which will be used in the proofs of Theorem \ref{Th:mainresult} and Theorem \ref{Th:mainresultW1p}. 
\vskip2mm
\begin{definition}\label{Def:Omega_li_epsilon}
For any $n \in \mathbb N_+$ and $|\bm l|_1 \leq n+d-1$, let $0<\varepsilon<2^{-2n}$, and
$$
\Omega_{\bm i, \varepsilon}^{\bm l}=\prod_{j=1}^d\Omega_{i_j,\varepsilon}^{l_j}=\prod_{j=1}^d\left[\frac{i_j-1}{2^{l_j}}, \frac{i_j+1}{2^{l_j}}-\varepsilon \cdot \bm1_{\{i_j<2^{l_j}-1\}}\right],
$$
for any $\bm i\in \mathcal{I}_{\bm l}$.
Then we define
$$
\Omega_\varepsilon=\bigcap_{|\bm l|_1 \leq n+d-1} \Omega_{ \varepsilon}^{\bm l},\quad \Omega_{ \varepsilon}^{\bm l}  =\bigcup_{\bm i \in \mathcal{I}_{\bm l}} \Omega_{\bm i, \varepsilon}^{\bm l}.
$$    
\end{definition}
%Following the approach of Lemma 5.3 in \cite{lu2021deep}, we construct a neural network capable of approximating polynomials over a larger bounded region. The proof of Lemma \ref{Le:polynomial_appro} can be found in subsection \ref{Sec:polynomial_appro_proof}.
%\begin{lemma}\label{Le:polynomial_appro}
%For any $W, L ,d\in \mathbb{N}_{+}$ with $d\geq 2$,  there exists a function $\phi\in \mathcal{NN}(9(W+1)+2d-1,14d^2L)$ such that
%$$
%|\phi(\bm{x})-\prod_{j=1}^d x_j| \leq 3^{5d+4} (W+1)^{-14 dL}
%$$
%for any %$\bm{x}=\left(x_1,x_2,\ldots,x_d\right)^{\top} \in[-3,3]^d $.
%\end{lemma}

\subsection{Proof of Theorem \ref{Th:mainresult}}

\subsubsection{The case \texorpdfstring{$m= 2$}{m = 2}}\label{Relucase1}
Let $n=\lceil2\log( 2WL)\rceil $ and $0<\varepsilon < 2^{-2n}$.
For any multi-index $\bm{l} \in \mathbb{N}_+^d$ satisfying $|\bm{l}|_1 \leq n+d-1$, we first construct a ReLU DNN function $\phi_{\bm{l}}$ that approximates
$$
f_{\bm l}(\bm x):=\sum_{\bm i \in \mathcal{I}_{\bm l}}v_{\bm l,\bm i}\phi_{\bm l,\bm i}(\bm x)
$$
efficiently over the region $\Omega_\varepsilon^{\bm{l}}$, and then use this result to construct $\phi$ with the desired approximation accuracy.
%and the uniform norm of $\phi_{\bm l}$ in the whole region $\Omega$ is finite.

\textbf{Step 1} Construct the desired function $\phi_{\bm l}$.

For any fixed multi-index $\bm{l} = (l_1,\ldots,l_d) \in \mathbb{N}_+^d$, applying Lemma \ref{Le:stepfunction} (with appropriate adjustments made to the parameters of the final affine transformation) with $K = 2^{l_j-1} \leq W^2(2L)^2$ yields a function $\phi_{l_j} \in \mathcal{NN}(4W+3, 8L+5)$ for each dimension $j=1,\ldots,d$, satisfying the following piecewise constant property:
\begin{equation}\label{Eq:phi_l_jstepfunction}
\phi_{l_j}(x)=2k-1 \quad \text { if } x \in\left[\frac{2k-2}{2^{l_j}}, \frac{2k}{2^{l_j}}-\varepsilon \cdot \mathbf{1}_{\{k < 2^{l_j}-1\}}\right],
\end{equation}     
for each integer $k\in\{1,2,\ldots,2^{l_j-1}\}$.
Define the vector-valued function
$$
\Phi_{\bm l}^1(\bm x)=\left(\phi_{l_1}(x_1),\phi_{l_2}(x_2),\ldots,\phi_{l_d}(x_d)\right)^{\top},\quad \text{for all}\,\, \bm x\in \Omega.
$$
By Lemma \ref{Le:concate}, we have $\Phi_{\bm l}^1\in \mathcal{NN}(4dW + 3d, 8L + 5)$. Moreover, by \eqref{Eq:phi_l_jstepfunction}, it holds that
\begin{equation}\label{Eq:Phi_l^1}
\Phi_{\bm l}^1(\bm x)= \bm i,\quad \text{
for any}\,\, \bm x\in \Omega_{\bm i,\varepsilon}^{\bm l}, \,\,\text{and for each}\,\,\bm i\in \mathcal{I}_{\bm l}.
\end{equation}
Let $h(\bm x)=\frac{\bm x-\bm 1}{2}$ for any $\bm x\in \mathbb R^d$, then $h$ is a bijection from $\mathcal{I}_{\bm l}$ to $\prod_{j=1}^d\{0,1,\ldots,2^{l_j-1}-1\}$. In addition, for each $k\in \big\{0,1,\ldots,\prod_{j=1}^d2^{l_j-1}-1\big\}$, we could define  a bijection $$
\bm \eta(k):=(\eta_1,\eta_2,\ldots,\eta_d)^{\top}\in \prod_{j=1}^d\{0,1,\ldots, 2^{l_j-1}-1\},
$$
such that $\sum_{j=1}^d\eta_j\prod_{r=1}^j2^{l_r-1}=k$.
Consider the affine transformation $\mathcal{A}_1(\bm{x}) = \bm{W}_1\bm{x}+\bm b_1$ where $$\bm{W}_1 \in \mathbb{R}^{1 \times d}:=\frac{1}{2}\left(2^{l_1-1},\prod_{r=1}^{2} 2^{l_r-1},\ldots,\prod_{r=1}^{d} 2^{l_r-1}\right)^{\top}$$
and $\bm b_1:=\frac{1}{2}\sum_{j=1}^d\prod_{r=1}^{j} 2^{l_r-1}$. Then the composite network
\begin{equation}\label{Eq:Phi_l}
\Phi_{\bm{l}} := \mathcal{A}_1 \circ \Phi_{\bm{l}}^1 \in \mathcal{NN}(4dW + 3d, 8L + 5)
\end{equation}
satisfies for any $\bm{x} \in \Omega^{\bm{l}}_{\bm{i},\epsilon}$,
\begin{equation}
\Phi_{\bm{l}}(\bm{x}) = \operatorname{ind}(\bm{i}) := \sum_{j=1}^d \frac{i_j-1}{2} \prod_{r=1}^{j} 2^{l_r-1}\in \left\{0,1,\ldots,\prod_{j=1}^d2^{l_j-1}-1\right\}.
\end{equation}
Consider $\mathcal{A}_2(\bm x):=\bm W_2\bm x$ with
$\bm W_2=[\bm I_d,\bm I_d]^{\top}$ for any $\bm x \in \mathbb R^{d}$, where $\bm I_d \in \mathbb R^{d\times d}$ is the identity matrix. This yields the vector duplication mapping
\begin{equation}\label{Eq:affine_A2}
\mathcal{A}_2(\bm x)=\binom{\bm x}{\bm x}
\end{equation}
for any $\bm x \in \Omega$.
Define the function $g_1$ by 
\begin{equation}\label{Eq:mapx_to_i}
g_1\binom{\bm z_1}{\bm z_2}=\left(\Phi_{\bm l}^1(\bm z_1),\sigma(\bm z_2)+\sigma(-\bm z_2)\right)^{\top},
\end{equation}
for any $\bm z_1,\bm z_2\in \mathbb R^d$. By Lemma \ref{Le:concate} and \eqref{Eq:Phi_l^1}, the composed network $g_1 \circ \mathcal{A}_2$ satisfies
\begin{equation}\label{Eq:g_1A_2size}
g_1\circ \mathcal{A}_2 \in \mathcal{NN}(4dW + 5d, 8L + 5).
\end{equation}
Denote by $\bm I^{(i)}\in \mathbb R^{3\times d}$ the matrix where the element in column $i$ is $1 $, while the elements in other columns are all $0 $  and $\bm H^{(i)}\in \mathbb R^{d\times 3}$ the matrix where the $i$-th row vector is $(1, -2, 1)$, while the elements in other rows are all $0 $, for each $i=1,2,\ldots, d$.
Let  $\mathcal{A}_3(\bm x_1)=\bm W_3\bm x_1+\bm b_3$ and $\mathcal{A}_4(\bm x_2)=\bm W_4\bm x_2$ with
$$\bm W_3=\begin{pmatrix}
    -\bm I^{(1)} & 2^{l_1}\bm I^{(1)} \\
     \vdots & \vdots \\ 
     -\bm I^{(d)} & 2^{l_d}\bm I^{(d)} \\ 
    \bm  O& \bm I_d
\end{pmatrix}\in \mathbb R^{4d\times 2d},\,\,
\bm W_4= \begin{pmatrix}
    \bm H^{(1)} & \cdots &\bm H^{(d)} &\bm O \\
     \bm O & \ldots &\bm O & \bm I_d\\ 
\end{pmatrix}\in \mathbb R^{2d\times 4d},
$$
and $\bm b_3=(-1,0,1,-1,0,1,\ldots,-1,0,1,\bm 0_{d}^{\top})^{\top}\in \mathbb R^{4d}$ for any $\bm x_1 \in \mathbb R^{2d},\bm x_2 \in \mathbb R^{4d}$, where $\bm I_d \in \mathbb R^{d\times d}$ is the identity matrix and $\bm 0_d=(0,0,\ldots,0)^{\top}\in \mathbb R^d$. 
It follows from \eqref{Eq:g_1A_2size} that
\begin{equation}\label{Eq:phi_1size}
\phi_1:=\mathcal{A}_4\circ\sigma\circ \mathcal{A}_3\circ g_1\circ \mathcal{A}_2\in \mathcal{NN}(4dW + 5d, 8L + 6).
\end{equation}
Since each function $\phi_{l_j,i_j}(x_j)$ can be represented as 
$$
\phi_{l_j,i_j}(x_j)=\sigma(2^{l_j}x_j-i_j-1)-2\sigma(2^{l_j}x_j-i_j)+\sigma(2^{l_j}x_j-i_j+1)
$$
for any $j=1,2,\ldots,d$, then we have
$$
\phi_1(\bm x)=\left(\begin{array}{c}
\phi_{l_1,i_1}(x_1) \\
\phi_{l_2,i_2}(x_2)\\
\vdots\\
\phi_{l_d,i_d}(x_d)\\
\bm x\\
\end{array}\right)
$$
for any $\bm x\in \Omega_{\bm i,\varepsilon}^{\bm l}$.
Let $g_2$ be the function satisfying 
\begin{equation}\label{Eq:g_2}
g_2\binom{\bm z_1}{\bm z_2}=\left(\sigma(\bm z_1)-\sigma(-\bm z_1),\Phi_{\bm l}(\bm z_2)\right)^{\top},
\end{equation}
for any $\bm z_1,\bm z_2\in \mathbb R^d$. By Lemma \ref{Le:concate} and equations \eqref{Eq:Phi_l} and \eqref{Eq:phi_1size}, we obtain \begin{equation}\label{Eq:tilde_phi}
\widetilde{\phi}_1:=g_2\circ \phi_1 \in \mathcal{NN}(4dW + 7d, 16L + 11)
\end{equation} 
such that for any $\bm x\in \Omega_{\bm i,\varepsilon}^{\bm l}$, we have
$$
\widetilde{\phi}_1(\bm x)=\left(\begin{array}{c}
\phi_{l_1,i_1}(x_1) \\
\phi_{l_2,i_2}(x_2)\\
\vdots\\
\phi_{l_d,i_d}(x_d)\\
\operatorname{ind}(\bm i)\\
\end{array}\right).
$$
Since 
$|v_{\bm l,\bm i}|\leq 1$ for any $\bm i\in \mathcal{I}_{\bm l}$ and $$\max_{\bm i\in \mathcal{I}_{\bm l}}|\operatorname{ind}(\bm i)|\leq \prod_{j=1}^d2^{l_j-1}-1\leq W^2(2L)^2-1,$$ by Lemma \ref{Le:fitsample}, there exists a function $$\phi_2\in \mathcal{NN}(16 s(W+1) \log(8 W),5(2L+2)\log(8 L))$$
such that $0 \leq \phi_2(x) \leq 1$ for any $x \in \mathbb{R}$,
satisfying
$$
\left|\phi_2(\operatorname{ind}(\bm i))-\frac{v_{\bm l,\bm i}+1}{2}\right|\leq W^{-2s}(2L)^{-2s}
$$
for any $\bm i\in \mathcal{I}_{\bm l}$. Consequently, setting $\widetilde{\phi}_2:=2\phi_2-1$, then
$$\widetilde{\phi}_2\in \mathcal{NN}(16 s(W+1) \log(8 W),5(2L+2)\log(8 L))$$ with the uniform upper bound $3$ in $\mathbb R$. Let $g_3$ be the function satisfying 
$$g_3\binom{\bm z_1}{z_2}=\left(\sigma(\bm z_1)-\sigma(-\bm z_1),\widetilde{\phi}_2( z_2)\right)^{\top},
$$
for any $\bm z_1\in \mathbb R^d, z_2\in \mathbb R$. Using Lemma \ref{Le:concate}, we can deduce that 
$$
g_3\in \mathcal{NN}(16 s(W+1) \log(8 W)+2d,5(2L+2)\log(8 L))  
$$
and
$$
\phi_3:= g_3 \circ \widetilde{\phi}_1\in \mathcal{NN}(16 ds(W+1) \log(8 W)+7d,(26L+21)\log(8 L))
$$
such that for any $\bm x\in \Omega_{\bm i,\varepsilon}^{\bm l}$, it follows that
$$
\phi_3(\bm x)=\left(\begin{array}{c}
\phi_{l_1,i_1}(x_1) \\
\phi_{l_2,i_2}(x_2)\\
\vdots\\
\phi_{l_d,i_d}(x_d)\\
\widetilde{v}_{\bm l,\bm i}\\
\end{array}\right),
$$
where $\widetilde{v}_{\bm l,\bm i}=\widetilde{\phi}_2(\operatorname{ind}(\bm i))$ and \begin{equation}\label{Eq:v_bound}
|\widetilde{v}_{\bm l,\bm i}-v_{\bm l,\bm i}|\leq 2W^{-2s}(2L)^{-2s},\quad \text{for any}\,\, \bm i\in \mathcal{I}_{\bm l}.
\end{equation}
By Lemma \ref{Le:polynomial_W1p}, there exists a function $$\phi_4\in \mathcal{NN}(15(W + 1) + 2d+1 ,14(d+1)^2L)$$ such that 
for $\bm{x} \in[-1,1]^{d}\times[-3,3]$,
$$
|\phi_4(\bm{x})-x_1x_2\ldots x_{d+1}| \leq 3^{5d+4}(W+1)^{-14 (d+1) L}.
$$
\begin{figure*}[ht]
\centering
\includegraphics[width=0.9\linewidth]{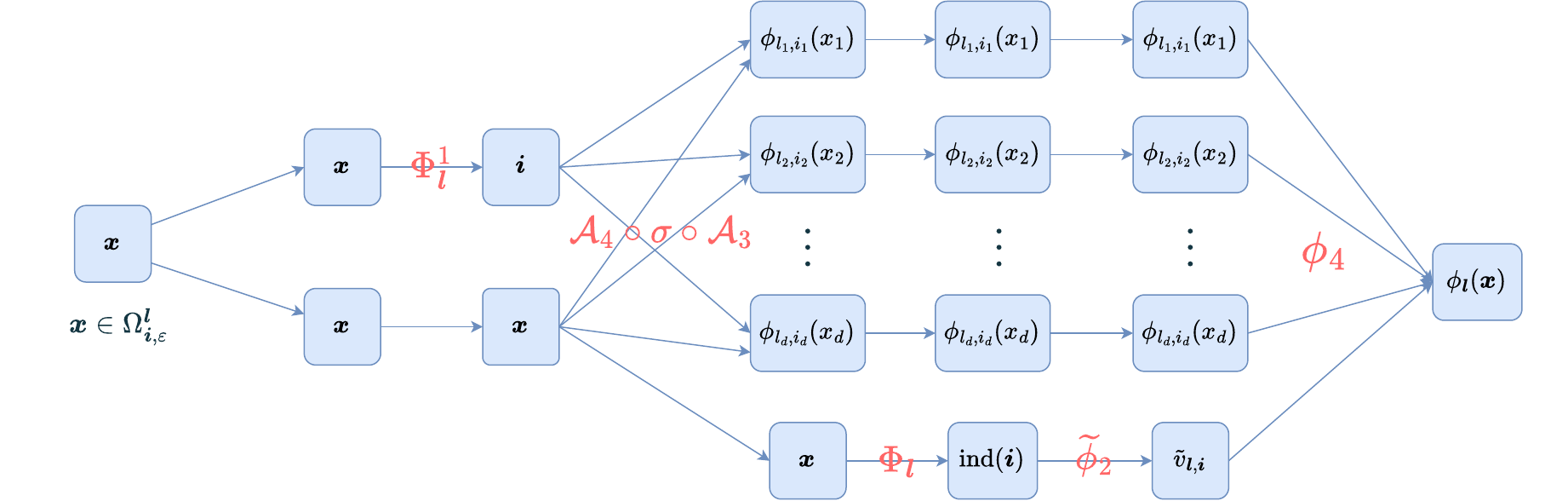}
\caption{An illustration of the network architecture implementing $\phi_{\bm l}$, where the input $\bm x$ belongs to a region $\Omega_{\bm i,\varepsilon}^{\bm l}$. Red functions represent the transformation from the left input vector to the right output vector. The $d$-dimensional identity mapping can be represented as a ReLU network with a width of $2d$ and a depth of $1$. Finally, the product network $\phi_4$ of $d+1$ variables implies the desired network function.}
\label{fig:phi_l_Lp}
\end{figure*}
Finally, let $\phi_{\bm l}:= \phi_4 \circ \phi_3$; see Figure \ref{fig:phi_l_Lp} for illustration. Using Lemma \ref{Le:concate}, we have 
$$
\phi_{\bm l}
\in \mathcal{NN}(16 ds(W+1) \log(8 W)+7d,(26L+21)\log(8 L)+14(d+1)^2L).
$$
Furthermore, by $|\widetilde{v}_{\bm l,\bm i}|\leq 3$ and $|\phi_{l_j,i_j}|\leq 1$ for $j=1,2,\ldots,d$, we can establish that
$$
|\phi_{\bm l}(\bm x)-\widetilde{v}_{\bm l,\bm i}\phi_{\bm l,\bm i}(\bm x)| \leq 3^{5d+4}(W+1)^{-14 (d+1) L},
$$
for any $\bm x\in \Omega_{\bm i,\varepsilon}^{\bm l}$ and
for each $\bm i\in \mathcal{I}_{\bm l}$. Combining this result with \eqref{Eq:v_bound} and $\|\phi_{\bm l,\bm i}\|_{L_{\infty}(\Omega)}\leq 1$ implies that
$$
\begin{aligned}
&|\phi_{\bm l}(\bm x)-f_{\bm l}(\bm x)|\\
\leq& \Big|\phi_{\bm l}(\bm x)-\sum_{\bm i \in \mathcal{I}_{\bm l}}\widetilde{v}_{\bm l,\bm i}\phi_{\bm l,\bm i}(\bm x)\Big| +\Big|\sum_{\bm i \in \mathcal{I}_{\bm l}}(v_{\bm l,\bm i}-\widetilde{v}_{\bm l,\bm i})\phi_{\bm l,\bm i}(\bm x)\Big|\\
\leq &3^{5d+4}(W+1)^{-14 (d+1) L}+ 2W^{-2s}(2L)^{-2s}|\mathcal{I}_{\bm l}|\\
\leq &3^{5d+4}(W+1)^{-14 (d+1) L}+ 2W^{-2s+2}(2L)^{-2s+2}
\end{aligned}
$$
for any $\bm x\in \Omega_{\varepsilon}^{\bm l}$.

%\textbf{Step 3} Estimate approximation error for $f_{\bm l}$.

%\textbf{Step 4} Determine the width and depth of the ReLU network implementing $\phi_{\bm l}$.

\textbf{Step 2} Construct the desired function $\phi$ and determine the size of the ReLU network implementing $\phi$ and the approximation error for $f$.

Given $1\leq p<\infty$, we have 
$$
\begin{aligned}
\|f_{\bm l}-\phi_{\bm l}\|_{L_p(\Omega)}^p
&=\|f_{\bm l}-\phi_{\bm l}\|_{L_p(\Omega_{\varepsilon})}^p+\|f_{\bm l}-\phi_{\bm l}\|_{L_p(\Omega\backslash\Omega_{\varepsilon})}^p\\
&\leq \left(3^{5d+4}(W+1)^{-14 (d+1) L}+ 2W^{-2s+2}(2L)^{-2s+2}\right)^p\\
&+\mu((\Omega\backslash\Omega_{\varepsilon}))(\|f_{\bm l}\|_{L_{\infty}(\Omega)}
+\|\phi_{\bm l}\|_{L_{\infty}(\Omega)})^p,
\end{aligned}
$$
where $\Omega_{\varepsilon}:=\cap_{|\bm l|_1\leq  n+d-1}\Omega_{\varepsilon}^{\bm l}$ and $\mu$ is the Lebesgue measure of a set in $\mathbb R^d$. In fact, we can calculate $\|f_{\bm l}\|_{L_{\infty}(\Omega)}
$ (see \cite{Bungartz_Griebel2004}) and $\|\phi_{\bm l}\|_{L_{\infty}(\Omega)}$ using Lemma 2.2 in \cite{Shen_Yang_Zhang2019} for each $\bm l$ satisfying $|\bm l|_1\leq n+d-1$. Notably, both these norms only depend on 
$n$. Consequently, by selecting a sufficiently small $\varepsilon$, we can ensure the following holds
$$
\begin{aligned}
\mu((\Omega\backslash\Omega_{\varepsilon}))^{\frac{1}{p}}(\|f_{\bm l}\|_{L_{\infty}(\Omega)}
+\|\phi_{\bm l}\|_{L_{\infty}(\Omega)})
\leq 3^{5d+4}(W+1)^{-14 (d+1) L}+ 2W^{-2s+2}(2L)^{-2s+2}.
\end{aligned}
$$
By setting $s=3$, we can deduce
\begin{equation}\label{Eq:phi_l_size}
\begin{aligned}
\phi_{\bm l}
&\in \mathcal{NN}(48d(W+1) \log(8 W)+7d,
(26L+21)\log(8 L)+14(d+1)^2L)\\
&\subset\mathcal{NN}(55d(W+1) \log(8 W),
40(d+1)^2(L+1)\log(8 L)),
\end{aligned}
\end{equation} 
and
\begin{equation}\label{Eq:phi_l_error}
\|f_{\bm l}-\phi_{\bm l}\|_{L_p(\Omega)}
\leq 3^{5d+6}(W+1)^{-14 (d+1) L}+ 2W^{-4}L^{-4}\leq 4W^{-4}L^{-4}
\end{equation}
for each $\bm l$ which satisfies $|\bm l|_1\leq n+d-1$. Define
$$
\phi:=\sum_{|\bm l|_1\leq n+d-1}\phi_{\bm l}.
$$
Then, utilizing Lemma \ref{Le:size_transition}, \eqref{Eq:phi_l_size} and the fact 
\begin{equation}\label{Eq:sigma_1_bound}
\begin{aligned}\sum_{|\bm l|_1\leq n+d-1}1 &=|\{\bm l\in \mathbb N_+^d:|\bm l|_1\leq n+d-1\}|\leq (n+d)^d\leq (2d)^d(\log(8W))^d(\log(8L))^d,
\end{aligned}
\end{equation}
where the first inequality is a classical result in the literature (see also \cite[Lemma 4.7]{li2025newrationalapproximationalgorithm}), the size of the network function satisfies
$$
\phi \in \mathcal{NN}(C_1W \log(8 W)^{d+1},C_2L\log(8 L)^{d+1}),
$$ 
where $C_1=112d(2d)^d$ and $C_2=320d^2$.
Finally, combining $n=\left\lceil2 \log( 2WL)\right\rceil$ with \eqref{Eq:phi_l_error} and \eqref{Eq:sparse_error} implies 
$$
\begin{aligned}
\|f-\phi\|_{L_p(\Omega)}&\leq \|f-\Pi_nf\|_{L_p(\Omega)}+\|\Pi_nf-\phi\|_{L_p(\Omega)} \\
&\leq \|f-\Pi_nf\|_{L_p(\Omega)}+\sum_{|\bm l|_1\leq n+d-1}\|f_{\bm l}-\phi_{\bm l}\|_{L_p(\Omega)}\\
&\leq 2^{-2n}(nd)^{3(d-1)}+4(n+d)^dW^{-4}L^{-4}\\
&\leq C_3 W^{-4}L^{-4}(\log( 2WL))^{3d},
\end{aligned}
$$
where $C_3= (2d)^{3d}$.
{  Substituting $W$ and $L$ with $\frac{W}{C_1 \log(8 W)^{d+1}}$ and $\frac{L}{C_2\log(8 L)^{d+1}}$, respectively, yields $
\phi \in \mathcal{NN}(W,L)
$ and the following error bound:
$$
\|f-\phi\|_{L_p(\Omega)}\leq \widetilde{C}_1W^{-4}L^{-4}(\log(8 W))^{7d+4}(\log(8 L))^{7d+4}
$$ where $\widetilde{C}_1=(C_1C_2)^4C_3$.
}
\subsubsection{The case \texorpdfstring{$m\geq 3$}{m >= 3}}\label{Relucase2}
For $m\geq 3$, let $n=\left\lceil2 \log( 2WL)\right\rceil$ and consider the higher order sparse grid interpolation  
\begin{equation}\label{Eq:interpolation_m}
\begin{aligned}
\Pi_n^mf(\bm{x})&=\sum_{|\bm l|_1 \leq n+d-1} \sum_{\bm i \in \mathcal{I}_{\bm l}} v_{\bm l, \bm i}^m\phi_{\bm l, \bm i}^m(\bm {x}):=\sum_{|\bm l|_1 \leq n+d-1}f_{\bm l}^m
.
\end{aligned}
\end{equation}
Given $1\leq p<\infty$, the interpolation error of $\Pi_n^mf$ has been discussed (see \cite{Bungartz_Griebel2004,lizhang2025korobov}) and is presented as follows
\begin{equation}\label{Eq:higherorder_interpolation_error}
\|f-\Pi_n^mf\|_{L_p(\Omega)} \leq C_{m,d}2^{-mn} n^d,
\end{equation}
where $C_{m,d}$ is a constant depending only on $m$ and $d$.
Furthermore, it is observed in \cite{lizhang2025korobov} that the basis function in \eqref{Eq:interpolation_m} could be written as the following form
$$
\phi_{\bm{l}, \bm{i}}^m(\bm{x})=\prod_{j=1}^d \prod_{k=1}^{m-1} \rho_{l_j,i_j,k}(x_j),
$$
where each $\rho_{l_j,i_j,k}(x_j)$ can be represented as $$ \sigma(a_{l_j,i_j,k}x_j+b_{l_j,i_j,k})$$ with $a_{l_j,i_j,k},b_{l_j,i_j,k}\in\mathbb{R}$ and satisfies $0\leq\rho_{l_j,i_j,k}(x_j)\leq 2^{n+d-1}$. We observe that for $m\geq 3$, the hierarchical Lagrange interpolation of $f\in X^m_p(\Omega)$ (see \cite{Bungartz_Griebel2004}) on sparse grid can result in at most 
$2^{m-3}
$ possibilities for the shape of the basis functions (see \cite[Section 4.2]{Bungartz_Griebel2004} and Figure 4.8 therein for instance). We choose the same shape in each support set of the basis function, e.g., the shape of the basis function at the position
$\bm i=\bm 1$. Since the $L_p$ norm of $\phi_{\bm l,\bm i}^m$ is independent of the position of $\bm i$, the 
$L_p$-norm of the changed basis functions on $
\Omega_{\bm
l,\bm i}$ and  the error of the sparse grid interpolation
remains the same. Henceforth, we will continue to denote the modified basis functions as $\phi_{\bm l,\bm i}^m$.
And, the product factors could be represented as  $$\rho_{l_j,i_j,k}(x_j)=\sigma(-2^{l_j}c_{k,j} x_j+c_{k,j}\bm i_j+1),$$ for $j=1,2,\ldots,d$ and for $k=1,2\ldots,m-1$. Here, some 
$c_{k,j}$ values are appropriately set to $0$ based on the degree of the basis functions. Since each $\rho_{l_j,i_j,k}(x_j)$ is a piecewise linear function activated by the ReLU function, we can define an affine map as follows
$$
\mathcal{B}_1\left(\binom{\bm x}{\bm x}\right)=\begin{pmatrix}
     \bm B_1 &\bm O \\
   \bm O& \bm I_d
\end{pmatrix}\binom{\bm x}{\bm x}+\bm b_4
$$ for any $\bm x\in \mathbb R^d$, where $\bm B_1\in \mathbb R^{(m-1)d\times d}$ and $\bm b_4 \in \mathbb R^{md}$ are determined by $c_{k,j}$ and $l_j$, satisfying for any $\bm x\in \Omega_{\bm i,\varepsilon}^{\bm l}$,
$$
\begin{aligned}
\sigma\circ \mathcal{B}_1\circ g_1\circ \mathcal{A}_2(\bm x)
=&\left(\psi_{l_1,i_1,1}(x_1),\ldots,\psi_{l_1,i_1,m-1}(x_1),\right.\\
&\left.\ldots,\psi_{l_d,i_d,1}(x_d),\ldots,\psi_{l_d,i_d,m-1}(x_d),\bm x^{\top}\right)^{\top},
\end{aligned}
$$
where $0\leq \psi_{l_j,i_j,k}=2^{-n-d+1}\rho_{l_j,i_j,k}\leq 1$ for any $j=1,2,\ldots,d$ and $k=1,2,\ldots,m-1$. Here $\mathcal{A}_2$ and $g_1$ are defined in \eqref{Eq:affine_A2} and \eqref{Eq:mapx_to_i} respectively, and $\sigma$ represents the ReLU function. By Proposition 4.1 in \cite{lu_shen_yang_zhang2021deep} with $d=m-1$ and $k=m$, replacing $L$ by $dL$, there exists a function 
$$P\in \mathcal{NN}(9(W+1)+m-1,14 m^2 dL)$$ such that
$$
|P(\bm{z})-z_1z_2\cdots z_{m-1}| \leq 9 m(W+1)^{-14 md L}
$$
for any $\bm{z}=(z_1,z_2,\ldots,z_{m-1}) \in[0,1]^{m-1}$.
Combining this result with Lemma \ref{Le:concate} and \eqref{Eq:g_1A_2size} implies that there exists a function 
$
\Theta_1\in \mathcal{NN}(21md W,28 m^2d L)
$ such that  for any $\bm x\in \Omega_{\bm i,\varepsilon}^{\bm l}$, the equality
$$
\begin{aligned}
\Theta_1(\bm x)
=&\left(\widetilde{\psi}_{l_1,i_1}(x_1),\psi_{l_2,i_2,1}(x_2),\ldots,\psi_{l_2,i_2,m-1}(x_2),\right.\\
&\left.\ldots,\psi_{l_d,i_d,1}(x_d),\ldots,\psi_{l_d,i_d,m-1}(x_d),\bm x^{\top}\right)^{\top}
\end{aligned}
$$ holds. Continuing this process $d-1$ times, we derive a function $\Theta\in  \mathcal{NN}(21md W,28 m^2 d^2L)$ such that for any $\bm x\in \Omega_{\bm i,\varepsilon}^{\bm l}$ we have
$$
\Theta(\bm x)=\left(\begin{array}{c}
\widetilde{\psi}_{l_1,i_1}(x_1)\\
\widetilde{\psi}_{l_2,i_2}(x_2)\\

\vdots\\
\widetilde{\psi}_{l_d,i_d}(x_d)\\
\bm x
\end{array}\right),
$$
and it satisfies
\begin{equation}\label{Eq:widetilde_psi_error}
\Big|\widetilde{\psi}_{l_j,i_j}(x_j)-\prod_{k=1}^{m-1}\psi_{l_j,i_j,k}(x_j)\Big|\leq 9 m(W+1)^{-14 m dL}
\end{equation}
and $
|\widetilde{\psi}_{l_j,i_j}(x_j)|\leq 2$,
for all $j=1,2,\ldots,d$. 
By following a similar rationale as in the case of 
$m=2$ with 
$
s=md+m$, we ultimately obtain
\begin{equation}\label{Eq:phi_l1_size}
\phi^0_{\bm l}
\in \mathcal{NN}(66mdW\log(8W),90m^2d^2L\log(8L))
\end{equation} 
and 
\begin{equation}\label{Eq:phi_l0_error}
\begin{aligned}
\Big|\phi^0_{\bm l}-\widetilde{v}^m_{\bm l,\bm i}\prod_{j=1}^d\widetilde{\psi}_{l_j,i_j}(x_j)\Big|
&\leq C_{m,d}(W+1)^{-14mdL}
\end{aligned}
\end{equation}
for $\bm x\in \Omega_{\varepsilon}^{\bm l}$, where $\widetilde{v}^m_{\bm l,\bm i}$ satisfies
\begin{equation}\label{Eq:v_li_1error}
|\widetilde{v}^m_{\bm l,\bm i}-v^m_{\bm l,\bm i}|\leq 2C_{m,d}W^{-2md-2m}(2L)^{-2md-2m}.
\end{equation}
Then define $\phi^1_{\bm l}:=2^{md(n+d-1)}\phi^0_{\bm l}$ for each $\bm l$ which satisfies $|\bm l|_1\leq n+d-1$. Combining  \eqref{Eq:widetilde_psi_error}, \eqref{Eq:phi_l0_error}, \eqref{Eq:v_li_1error}, Lemma \ref{Le:prod_approx_difference}, and the inequality $2^{md(n+d-1)}\leq C_{m,d}W^{2md}(2L)^{2md}$ implies that
\begin{equation}\label{Eq:phi_1l_error}
\begin{aligned}
&\Big|\phi_{\bm l}^1-\sum_{\bm i \in \mathcal{I}_{\bm l}} v_{\bm l, \bm i}^m\phi_{\bm l, \bm i}^m(\bm {x})\Big|\\
\leq &2^{md(n+d-1)}\left( \Big|\phi_{\bm l}^0-\widetilde{v}^m_{\bm l,\bm i}\prod_{j=1}^d\widetilde{\psi}_{l_j,i_j}(x_j)\Big|+\Big|\widetilde{v}^m_{\bm l,\bm i}\prod_{j=1}^d\widetilde{\psi}_{l_j,i_j}(x_j)-v^m_{\bm l,\bm i}\prod_{j=1}^d \prod_{k=1}^{m-1} \psi_{l_j,i_j,k}(x_j)\Big|\right)\\
\leq &C_{m,d}W^{2md}(2L)^{2md}\cdot \Big((W+1)^{-14mdL}+\left.2W^{-2md-2m}(2L)^{-2md-2m}\right.\\
&+\left.\Big|\prod_{j=1}^d\widetilde{\psi}_{l_j,i_j}(x_j)-\prod_{j=1}^d\prod_{k=1}^{m-1}\psi_{l_j,i_j,k}(x_j)\Big|\right)\\
\leq &C_{m,d}W^{-2m}L^{-2m}
\end{aligned}
\end{equation}
for any $\bm x\in \Omega_{\varepsilon}^{\bm l}$.
Define
$$
\phi^1:=\sum_{|\bm l|_1\leq n+d-1}\phi^1_{\bm l}.
$$
By utilizing Lemma \ref{Le:size_transition}, \eqref{Eq:sigma_1_bound}, and \eqref{Eq:phi_l1_size}, the size of the network function $\phi^1$ satisfies
$$
\phi^1 \in \mathcal{NN}(C_7W(\log(8W))^{d+1},C_8L(\log(8L))^{d+1})
$$ 
where $C_7=70md^{d+1}$, $C_8=90m^2d^22^d$.
The remaining process involves estimating the error bounds of neural networks in the $L_p$ norm. In this scenario, we can transfer the $L_p$ integral of the error from the region $\Omega_{\varepsilon}^{\bm l}$ to the entire region $\Omega$ using the same method as in the case where $m=2$.
Finally, by 
combining \eqref{Eq:higherorder_interpolation_error} and \eqref{Eq:phi_1l_error} and using $n=\left\lceil2 \log( 2WL)\right\rceil$, we arrive at 
$$
\begin{aligned}
\|f-\phi^1\|_{L_p(\Omega)}&\leq \|f-\Pi_n^mf\|_{L_p(\Omega)}+\|\Pi_n^mf-\phi^1\|_{L_p(\Omega)} \\
&\leq C_{m,d}2^{-mn} n^d+\sum_{|\bm l|_1\leq n+d-1}\|f_{\bm l}^m-\phi^1_{\bm l}\|_{L_p(\Omega)}\\
&\leq C_{m,d}W^{-2m}L^{-2m}(\log( 2WL))^{d}
\end{aligned}
$$
where $C_{m,d}$ is a constant depending only on $m$ and $d$. {  Substituting $W$ and $L$ with $\frac{W}{C_7 \log(8 W)^{d+1}}$ and $\frac{L}{C_8\log(8 L)^{d+1}}$, respectively, yields $
\phi \in \mathcal{NN}(W,L)
$ and the following error bound:
$$
\|f-\phi^1\|_{L_p(\Omega)}\leq C_{m,d}W^{-2m}L^{-2m}(\log(8 W))^{2(m+2)(d+1)}(\log(8 L))^{2(m+2)(d+1)}
$$ where $C_{m,d}$ is a constant depending only on $m$ and $d$.
}
Thus we finish the proof.
\vspace{2mm}

\begin{remark}
It is worth noting that when $p=\infty$, the term $\mu((\Omega\backslash\Omega_{\varepsilon}))^{1/p}(\|f_{\bm l}\|_{L_{\infty}(\Omega)}
+\|\phi_{\bm l}\|_{L_{\infty}(\Omega)})$ cannot be rendered sufficiently small simply by reducing $\varepsilon$.  Consequently, the approximation error on the entire domain $\Omega$ cannot be directly controlled by the error on $\Omega_{\varepsilon}$ as formulated in \eqref{Eq:phi_l_error}. To address this for $f\in X^m_{\infty}(\Omega)$, we can adopt an alternative approach as established in \cite{Yang_Lu2024,lu_shen_yang_zhang2021deep} and Theorem 2.1 of \cite{lu_shen_yang_zhang2021deep} to derive the corresponding approximation rates.   
\end{remark}

\section{Approximation by ReLU DNN in \texorpdfstring{$W^1_p$}{W1p} Norm}\label{Sec:ReLU_DNNW1p}
The proof of Theorem \ref{Th:mainresultW1p} relies on several concepts and lemmas listed below. %The proof of Proposition \ref{Le:patch_psik_error} can be found in subsection \ref{Sec:proposition}. %To establish Proposition \ref{Le:patch_psik_error} and Theorem \ref{Th:mainresultW1p}, we introduce two crucial definitions initially presented in \cite{YangYangXiang2023}. 
\vskip2mm
\begin{definition}\label{Def:Omega_k}
Given $K,d \in \mathbb N_+$, for any $\bm k = (k_1,k_2,\ldots,k_d)\in \{1,2\}^d$, define
$$
\Omega_{\bm k}:=\prod_{j=1}^d\Omega_{k_j},
$$
where 
$\Omega_1=\prod_{i=0}^{K-1}\left[\frac{i}{K},\frac{i}{K}+\frac{3}{4K}\right]$ and $\Omega_2=\prod_{i=0}^{K}\left[\frac{i}{K}-\frac{1}{2K},\frac{i}{K}+\frac{1}{4K}\right]\cap[0,1]
$; see Figure \ref{fig:g_omega} for illustration.   
\end{definition}
\vskip2mm
\begin{definition}\label{Def:g_k}
Given $K, d \in \mathbb{N}_{+}$, define $g_{\bm{k}}(\bm{x})=\prod_{j=1}^d g_{k_j}(x_j)$ for all $\bm x\in \mathbb R^d$ and for each $\bm{k}=\left(k_1, k_2, \ldots, k_d\right) \in\{1,2\}^d$, where
\begin{equation}\label{Eq:g_k}
g_1(x):=\left\{\begin{array}{ll}
1, & x \in\left[\frac{i}{K}+\frac{1}{4 K}, \frac{i}{K}+\frac{1}{2 K}\right] \\
0, & x \in\left[\frac{i}{K}+\frac{3}{4 K}, \frac{i+1}{K}\right] \\
4 K\left(x-\frac{i}{K}\right), & x \in\left[\frac{i}{K}, \frac{i}{K}+\frac{1}{4 K}\right] \\
-4 K\left(x-\frac{4i+3}{4K}\right), & x \in\left[\frac{i}{K}+\frac{1}{2 K}, \frac{i}{K}+\frac{3}{4 K}\right]
\end{array}, 
\right.
\end{equation}
and
$
g_2(x):=g_1\left(x+\frac{1}{2 K}\right)
$; see Figure \ref{fig:g_omega} for illustration.
\end{definition}
\vskip2mm
\begin{lemma}[Proposition 1 from \cite{YangYangXiang2023}]\label{Le:phi_kg_k_error}
Given any $W, L, n ,d\in \mathbb{N}_{+}$ and $K=W^2L^2$, then for any $\bm{k}=\left(k_1, k_2, \ldots, k_d\right) \in\{1,2\}^d$, there is a function $$\phi_{\bm{k}} \in \mathcal{NN}((9+d)(W+1)+d-1,15 d(d-1) n L)$$ such that $$\left\|\phi_{\bm{k}}(\bm{x})-g_{\bm{k}}(\bm{x})\right\|_{W^1_{ \infty}\left(\Omega\right)} \leq 50 d^{\frac{5}{2}}(W+1)^{-4 d n L},$$   
where $g_{\bm k}$ is determined by $K$ and $d$ as specified in \eqref{Def:g_k}.
\end{lemma}

\begin{figure}
    \centering   \includegraphics[width=0.6\linewidth]{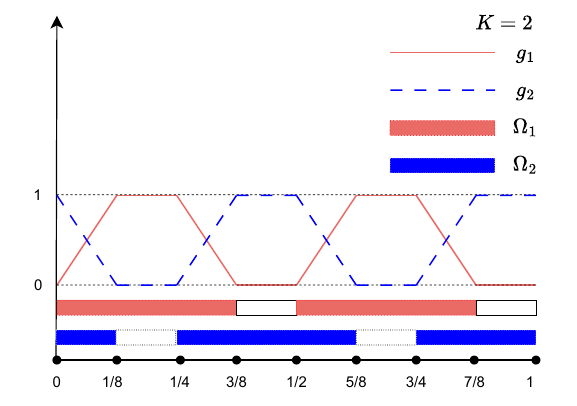}
    \caption{An illustration of $g_1,g_2,\Omega_1,$ and $\Omega_2$, where $K=2$.}
    \label{fig:g_omega}
\end{figure}

\begin{lemma}[Lemma 7 from \cite{YangYangXiang2023}]\label{Le:product_W1p_to_Omegak}
For any $g(\bm{x}) \in W^1_{\infty}(\Omega)$ and $\phi_{\bm{k}}(\bm{x})$ derived from Lemma \ref{Le:phi_kg_k_error} with $\bm{k}=\left(k_1, k_2, \ldots, k_d\right) \in\{1,2\}^d$, let
$
M=\max \left\{\|g\|_{W^1_{\infty}\left(\Omega\right)},\left\|\phi_{\bm{k}}\right\|_{W^1_{\infty}\left(\Omega\right)}\right\}
$,
then we have
$$
\left\|\phi_{\bm{k}}(\bm{x}) \cdot g(\bm{x})\right\|_{W^1_{\infty}\left(\Omega\right)}  =\left\|\phi_{\bm{k}}(\bm{x}) \cdot g(\bm{x})\right\|_{W^1_{\infty}\left(\Omega_{\bm{k}}\right)}
$$
and
$$
\left\|\phi_{\bm{k}}(\bm{x}) \cdot g(\bm{x})-\phi_M\left(\phi_{\bm{k}}(\bm{x}), g(\bm{x})\right)\right\|_{W^1_{\infty}\left(\Omega\right)}  =\left\|\phi_{\bm{k}}(\bm{x}) \cdot g(\bm{x})-\phi_M\left(\phi_{\bm{k}}(\bm{x}), g(\bm{x})\right)\right\|_{W^1_{\infty}\left(\Omega_{\bm{k}}\right)}
$$
where  $\phi_M$ is from Lemma \ref{Le:product_W1p_error} by choosing $a=M$.    
\end{lemma}
\vskip2mm
\begin{proposition}\label{Le:patch_psik_error}
Given any $f \in X^m_p(\Omega)$ with  $1\leq p<\infty$ and $m\geq 2$, for any $W,L \in \mathbb{N}_{+}$ and $\bm k=(k_1,k_2,\ldots,k_d)\in \{1,2\}^d$, there is a function $\psi_{\bm k}\in \mathcal{NN}(C_4W(\log W)^{d+1},C_5L(\log L)^{d+1})$ such that

\begin{equation*}
\begin{aligned}
&\|\psi_{\bm k}-f\|_{W^1_p(\Omega_{\bm k})} \leq C_{m,d} W^{-2(m-1)}L^{-2(m-1)}(\log( 2WL))^{d}, \\
&\|\psi_{\bm k}-f\|_{L_p(\Omega_{\bm k})} \leq C_{m,d} W^{-2m}L^{-2m}(\log( 2WL))^{3d},
\end{aligned}
\end{equation*}
where $C_4=68 d(2d)^dm^2$, $C_5=89d^2 m^2$, and $C_{m,d}$ is a constant depending only on $m$ and $d$.     
\end{proposition}
\subsection{Proof of Theorem \ref{Th:mainresultW1p}}
In this subsection, we provide the proof of Theorem \ref{Th:mainresultW1p} by using Lemma \ref{Le:product_W1p_error}, Lemma \ref{Le:phi_kg_k_error}, Lemma \ref{Le:product_W1p_to_Omegak}, and Proposition \ref{Le:patch_psik_error}.
%\subsubsection{The case \texorpdfstring{$m= 2$}{m = 2}}\label{ReluW1pthcase1}
%\newpage
%\subsubsection{The case \texorpdfstring{$m\geq 3$}{m >= 3}}\label{ReluW1pthcase2}
\vskip2mm
\begin{proof}
Given $m\geq 2$, by utilizing Lemma \ref{Le:phi_kg_k_error} and Proposition \ref{Le:patch_psik_error}, there exists a sequence of $$\psi_{\bm k}\in \mathcal{NN}(C_4W(\log W)^{d+1},C_5L(\log L)^{d+1})$$ and $$\phi_{\bm{k}} \in \mathcal{NN}((9+d)(W+1)+d-1,15 d(d-1) m L)$$ such that
\begin{equation}\label{Eq:tripple_Omega_error}
\begin{aligned}
&\|\psi_{\bm k}-f\|_{W^1_p(\Omega_{\bm k})}\leq C_{m,d} W^{-2(m-1)}L^{-2(m-1)}(\log( 2WL))^{d},\\
&\|\psi_{\bm k}-f\|_{L_p(\Omega_{\bm k})} \leq C_{m,d} W^{-2m}L^{-2m}(\log( 2WL))^{3d},\\
&\|\phi_{\bm{k}}(\bm{x})-g_{\bm{k}}(\bm{x})\|_{W^1_{ \infty}\left(\Omega\right)} \leq 50 d^{\frac{5}{2}}(W+1)^{-4 d m L},
\end{aligned}
\end{equation}
for each $\bm k \in\{1,2\}^d$. Using Lemma \ref{Le:product_W1p_error} with $a=C_{m,d}+50d^{\frac{5}{2}}+1$, there exists a function $\widetilde{\phi}\in \mathcal{NN}(15 (W+1),8mL)$ such that
$$
\|\widetilde{\phi}(x, y)-x y\|_{W^1_{ \infty}\left((-a, a)^2\right)} \leq 6 a^2 (W+1)^{-4mL}.
$$
Define
$$
\phi:=\sum_{\bm k\in\{1,2\}^d}\widetilde{\phi}(\phi_{\bm k},\psi_{\bm k}).
$$
Then, $\phi\in \mathcal{NN}(C_6W(\log W)^{d+1},C_7L(\log L)^{d+1})$ with $C_6=2^d(C_4+21d)$ and $C_7=C_5+8$.
Since 
$$
\sum_{\bm k\in\{1,2\}^d}g_{\bm k}(\bm x)\equiv1,\quad\text{for all}\quad \bm x\in \Omega,
$$
we have
$$
\begin{aligned}
\|f-\phi\|_{W^1_p(\Omega)}&=\Big\|\sum_{\bm k\in\{1,2\}^d}\left(fg_{\bm k}-\widetilde{\phi}(\phi_{\bm k},\psi_{\bm k})\right)\Big\|_{W^1_p(\Omega)}\\
&\leq \Big\|\sum_{\bm k\in\{1,2\}^d}\left(fg_{\bm k}-\phi_{\bm k}\cdot\psi_{\bm k}\right)\Big\|_{W^1_p(\Omega)}+\Big\|\sum_{\bm k\in\{1,2\}^d}\left(\phi_{\bm k}\cdot\psi_{\bm k}-\widetilde{\phi}(\phi_{\bm k},\psi_{\bm k})\right)\Big\|_{W^1_p(\Omega)}\\
&:=\mathcal{E}_1+\mathcal{E}_2.
\end{aligned}
$$
As for $\mathcal{E}_1$, by equation \eqref{Eq:tripple_Omega_error}, Lemma \ref{Le:product_W1p_to_Omegak}, and the fact 
$$
\begin{aligned}
\|\phi_{\bm k}\|_{W^1_{\infty}(\Omega_{\bm k})}
&\leq \|\phi_{\bm k}\|_{W^1_{\infty}(\Omega)}\leq\left\|\phi_{\bm{k}}(\bm{x})-g_{\bm{k}}\right\|_{W^1_{ \infty}(\Omega)}+\|g_{\bm{k}}\|_{W^1_{ \infty}(\Omega)}\leq C_dW^2L^2
\end{aligned}
$$ 
and 
$$
\|\phi_{\bm k}\|_{L_{\infty}(\Omega_{\bm k})}\leq \|\phi_{\bm k}\|_{L_{\infty}(\Omega)}\leq C_d,
$$
we have
$$
\begin{aligned}
\mathcal{E}_1
&\leq \sum_{\bm k\in\{1,2\}^d}\left(\|f\cdot\left(g_{\bm k}-\phi_{\bm k}\right)\|_{W^1_p(\Omega)}+\|\left(f-\psi_{\bm k}\right)\cdot \phi_{\bm k}\|_{W^1_p(\Omega)}\right)\\
&=\sum_{\bm k\in\{1,2\}^d}\left(\|f\cdot\left(g_{\bm k}-\phi_{\bm k}\right)\|_{W^1_p(\Omega)}+\|\left(f-\psi_{\bm k}\right)\cdot \phi_{\bm k}\|_{W^1_p(\Omega_{\bm k})}\right)\\
&\leq \sum_{\bm k\in\{1,2\}^d}C_{p,d}\left(\|f\|_{W^1_p(\Omega)}\|g_{\bm k}-\phi_{\bm k}\|_{W^1_{\infty}(\Omega)}+\|f-\psi_{\bm k}\|_{W^1_p(\Omega_{\bm k})} \|\phi_{\bm k}\|_{L_{\infty}(\Omega_{\bm k})}\right.\\
&\quad\left.+\|f-\psi_{\bm k}\|_{L_{p}(\Omega_{\bm k})} \|\phi_{\bm k}\|_{W^1_{\infty}(\Omega_{\bm k})}\right)\\
&\leq \sum_{\bm k\in\{1,2\}^d}C_{p,d}\left(\|f\|_{W^1_p(\Omega)}50 d^{\frac{5}{2}}(W+1)^{-4 d m L}+ W^{-2(m-1)}L^{-2(m-1)}(\log( 2WL))^{3d}\right)\\
&\leq C_{p,d} W^{-2(m-1)}L^{-2(m-1)}(\log( 2WL))^{3d}.
\end{aligned}
$$
As for $\mathcal{E}_2$, by equation \eqref{Eq:tripple_Omega_error} and Lemma \ref{Le:product_W1p_to_Omegak}, we have
$$
\begin{aligned}
\mathcal{E}_2
&\leq \sum_{\bm k\in\{1,2\}^d}\|\phi_{\bm k}\cdot\psi_{\bm k}-\widetilde{\phi}(\phi_{\bm k},\psi_{\bm k})\|_{W^1_p(\Omega)}\\
&\leq \sum_{\bm k\in\{1,2\}^d}\|\phi_{\bm k}\cdot\psi_{\bm k}-\widetilde{\phi}(\phi_{\bm k},\psi_{\bm k})\|_{W^1_p(\Omega_{\bm k})}\\
&\leq C_{m,d} \sum_{\bm k\in\{1,2\}^d}(W+1)^{-4mL} \|\phi_{\bm k}\|_{W^1_p(\Omega_{\bm k})}\|\psi_{\bm k}\|_{W^1_p(\Omega_{\bm k})}\\
&\leq C_{m,d} \sum_{\bm k\in\{1,2\}^d}(W+1)^{-4mL} W^2L^2 \\
&\leq C_{m,d} W^{-2(m-1)}L^{-2(m-1)}. 
\end{aligned}
$$
Thus, we obtain the error bound
$$
\|f-\phi\|_{W^1_p(\Omega)} \leq C_{m,p,d} W^{-2(m-1)}L^{-2(m-1)}(\log( 2WL))^{3d}.
$$
Substituting $W$ and $L$ with $\frac{W}{C_6 (\log 8W)^{d+1}}$ and $\frac{L}{C_7(\log 8L)^{d+1}}$, respectively, yields $
\phi \in \mathcal{NN}(W,L)
$ and the following error bound:
$$
\|f-\phi\|_{L_p(\Omega)}\leq \widetilde{C}_2W^{-2(m-1)}L^{-2(m-1)}(\log(8 W))^{2(m+1)(d+1)}(\log(8 L))^{2(m+1)(d+1)}
$$ where $\widetilde{C}_2$ is a constant depending only on $m,p$ and $d$.
 The proof is complete.
\end{proof}

\vskip2mm
In the remaining part of this section, we will prove Proposition \ref{Le:patch_psik_error} below.
\subsection{Proof of Proposition \ref{Le:patch_psik_error}}\label{Sec:proposition}
Let $n=\lceil2\log( 2WL)\rceil $ and $0<\varepsilon < 2^{-2n}$.  Without loss of generality, we only need to prove the case  $\bm k=(1,1,\ldots,1)$ with $K=(2W)^2L^2$ in the definition of $\Omega_{\bm k}$. Then for any multi-index $\bm l=(l_1,l_2,\ldots,l_d) \in \mathbb{N}_+^d$ satisfying $|\bm{l}|_1 \leq n+d-1$, we have $\Omega_{\bm k}\subset \Omega^{\bm l}_{\varepsilon}$ and we denote $\Omega_{\bm l,\bm i}$ by
$$
\Omega_{\bm l,\bm i}=\prod_{j=1}^d\left[\frac{i_j-1}{2^{l_j}},\frac{i_j+1}{2^{l_j}}\right],
$$
for each $\bm i=(i_1,i_2,\ldots,i_d)\in \mathcal{I}_{\bm l}$ which is defined in \eqref{Eq:i_l}. 
\subsubsection{The case \texorpdfstring{$m= 2$}{m = 2}}\label{ReluW1pcase1}
We first need to construct the desired function $\psi_{\bm k}$ to efficiently approximate the interpolation function
$$
\Pi_nf(\bm x)=\sum_{|\bm{l}|_1\leq n+d-1} \sum_{\bm{i} \in\mathcal{I}_{\bm l}} v_{\bm{l}, \bm{i}} \phi_{\bm{l}, \bm{i}}(\bm x)
$$
within the region $\Omega_{\bm k}$. Subsequently, combining this result with the error bound of $f$ and 
$\Pi_nf$ in the $W^1_p(\Omega)$  norm will lead to the desired outcomes.

\textbf{Step 1} Construct the desired function $\psi_{\bm k}$ to { efficiently approximate} the interpolation function $\Pi_nf$.

To this end, we first construct a function $\psi_{\bm k}^{\bm l}$ for
$$
f_{\bm l}(\bm x)=\sum_{\bm{i} \in\mathcal{I}_{\bm l}} v_{\bm{l}, \bm{i}} \phi_{\bm{l}, \bm{i}}(\bm x).
$$
It has been proved in Theorem \ref{Th:mainresult} that there exists a function $$
\phi_1
\in \mathcal{NN}(16 ds(W+1) \log(8 W)+7d,(26L+21)\log(8 L))
$$
such that for any $\bm x\in \Omega_{\bm i,\varepsilon}^{\bm l}$, it follows that
$$
\phi_1(\bm x)=\left(\begin{array}{c}
\phi_{l_1,i_1}(x_1) \\
\phi_{l_2,i_2}(x_2)\\
\vdots\\
\phi_{l_d,i_d}(x_d)\\
\widetilde{v}_{\bm l,\bm i}\\
\end{array}\right),
$$
where $
|\widetilde{v}_{\bm l,\bm i}-v_{\bm l,\bm i}|\leq 2W^{-2s}(2L)^{-2s}$ and $|\widetilde{v}_{\bm l,\bm i}|\leq 2$ for any $\bm i\in \mathcal{I}_{\bm l}$. By Lemma \ref{Le:polynomial_W1p} with $c=2$, there exists a function $\phi_2\in (15(W+1)+2d-1,28 d^2 L)$ such that $\|\phi_2\|_{W^1_{ \infty}\left([0,1]^{d}\times [-2,2]\right)} \leq 12\cdot 2^{d+1}$ and
\begin{equation}\label{Eq:phi_2_m2_W1perror}
\left\|\phi_2(\bm{x})-x_1 x_2 \cdots x_{d+1}\right\|_{W^1_{ \infty}\left([0,1]^{d}\times [-2,2]\right)} \leq 14 \cdot 2^{4d+4}(W+1)^{-14d L}.
\end{equation}
Define $\phi^{\bm l}_{\bm k}:=\phi_2\circ \phi_1$. By Lemma \ref{Le:concate},  $$\phi^{\bm l}_{\bm k} \in \mathcal{NN}(39sdW\log(8W),75d^2L\log(8L)).$$ 
When $s=3$ and utilizing equation \eqref{Eq:phi_2_m2_W1perror} along with the the inequality $\|\phi_{\bm{l}, \bm{i}}(\bm x)\|_{W^1_{\infty}(\Omega_{\bm i,\varepsilon}^{\bm l})}\leq W^{2}L^{2}$ implies
$$
\begin{aligned}
&\|\phi^{\bm l}_{\bm k}(\bm x)-f_{\bm l}(\bm x)\|_{W^1_{\infty}(\Omega_{\bm k})}\\
\leq &\|\phi^{\bm l}_{\bm k}(\bm x)-f_{\bm l}(\bm x)\|_{W^1_{\infty}(\Omega_{\varepsilon}^{\bm l})}\\
\leq &\max_{\bm i\in \mathcal{I}_{\bm l}}\|\phi^{\bm l}_{\bm k}(\bm x)-v_{\bm{l}, \bm{i}} \phi_{\bm{l}, \bm{i}}(\bm x)\|_{W^1_{\infty}(\Omega_{\bm i,\varepsilon}^{\bm l})}\\
\leq &\max_{\bm i\in \mathcal{I}_{\bm l}}\left(\|\phi^{\bm l}_{\bm k}(\bm x)-\widetilde{v}_{\bm{l}, \bm{i}} \phi_{\bm{l}, \bm{i}}(\bm x)\|_{W^1_{\infty}(\Omega_{\bm i,\varepsilon}^{\bm l})}+|\widetilde{v}_{\bm{l}, \bm{i}} -v_{\bm{l}, \bm{i}} |\|\phi_{\bm{l}, \bm{i}}(\bm x)\|_{W^1_{\infty}(\Omega_{\bm i,\varepsilon}^{\bm l})}\right) \\
\leq &\max_{\bm i\in \mathcal{I}_{\bm l}}\left(\max\left(1,\max_{1\leq j\leq d}|\phi_{l_j,i_j}(x_j)|_{W^1_{\infty}(\Omega_{i_j,\varepsilon}^{l_j})}\right)\right)\cdot C_d(W+1)^{-14dL}\\
&+2W^{-2s}(2L)^{-2s}\max_{\bm i\in \mathcal{I}_{\bm l}}\|\phi_{\bm{l}, \bm{i}}(\bm x)\|_{W^1_{\infty}(\Omega_{\bm i,\varepsilon}^{\bm l})}\\
=&\max_{1\leq j\leq d}2^{l_j}\cdot \left(C_d(W+1)^{-14dL}+2W^{-2s}L^{-2s}\right)\\
\leq &C_d(2W)^2L^2\cdot \left((W+1)^{-14dL}+2W^{-2s}L^{-2s}\right)\\
\leq &C_d W^{-4}L^{-4}.
\end{aligned}
$$
Define $\psi_{\bm k}=\sum_{|\bm{l}|_1\leq n+d-1}\phi^{\bm l}_{\bm k}$. According to Lemma \ref{Le:size_transition}, we have that $$\psi_{\bm k}\in \mathcal{NN}(C_4W(\log(8W))^{d+1},C_5L(\log(8L))^{d+1})$$ and
we obtain 
\begin{equation}\label{Eq:psi_k_f_n}
\begin{aligned}
\|\psi_{\bm k}-\Pi_nf\|_{W^1_{\infty}(\Omega_{\bm k})} &\leq \sum_{|\bm{l}|_1\leq n+d-1}\|\phi^{\bm l}_{\bm k}(\bm x)-f_{\bm l}(\bm x)\|_{W^1_{\infty}(\Omega_{\bm k})}\\
&\leq C_d\cdot (\log( 2WL))^d W^{-4}L^{-4},
\end{aligned}
\end{equation}
where $C_4=117d(2d)^{d}$ and $C_5=75d^2$.

\textbf{Step 2} Calculate the error bound of $f$ and 
$\Pi_nf$ in the $W^1_p(\Omega)$  norm.

We first note that
$$
\|f-\Pi_nf\|_{W^1_p(\Omega)}\leq\sum_{|\bm{l}|_1> n+d-1} \Big\|\sum_{\bm{i} \in\mathcal{I}_{\bm l}} v_{\bm{l}, \bm{i}} \phi_{\bm{l}, \bm{i}}(\bm x)\Big\|_{W^1_p(\Omega)}.
$$
By Lemma 2.1 in \cite{Bungartz_Griebel2004} and the fact that
$$
|v_{\bm l,\bm i}|\leq  2^{-(2-1/p)|\bm{l}|_1-d}\left\|D^{\bm 2}f\right\|_{L_p(\Omega_{\bm l,\bm i})}, 
$$
we can derive that
$$
\begin{aligned}
\Big\|\sum_{\bm{i} \in\mathcal{I}_{\bm l}} v_{\bm{l}, \bm{i}} \phi_{\bm{l}, \bm{i}}(\bm x)\Big\|_{W^1_p(\Omega)}^p &=\sum_{\bm{i} \in\mathcal{I}_{\bm l}} |v_{\bm{l}, \bm{i}}|^p\| \phi_{\bm{l}, \bm{i}}(\bm x)\|_{W^1_p(\Omega_{\bm l,\bm i})}^p\\
&\leq C_{p,d}\left(\sum_{j=1}^d2^{l_jp}\right) \cdot\sum_{\bm{i} \in\mathcal{I}_{\bm l}} |v_{\bm{l}, \bm{i}}|^p\|\phi_{\bm l,\bm i}(\bm x)\|_{L_p(\Omega_{\bm l,\bm i})}^p\\
&\leq C_{p,d}\left(\sum_{j=1}^d2^{l_jp}\right)\cdot 2^{-2p|\bm l|_{1}}.
\end{aligned}
$$
Thus, combining this result with the norm equivalence between $l_p$ and $l_2$ in $\mathbb R^d$ and the inequality (proof of  \cite[Theorem 3.8]{Bungartz_Griebel2004}) 
$$
\sum_{|\bm {l}|_1=i}\left(\sum_{j=1}^d 4^{l_j}\right)^{1 / 2} \leq d \cdot 2^i,
$$
we proceed as follows:
\begin{equation}\label{Eq:f_fn_W1p_error}
\begin{aligned}
\|f-\Pi_nf\|_{W^1_p(\Omega)}&\leq C_{p,d}\sum_{|\bm{l}|_1> n+d-1} \left(\sum_{j=1}^d2^{l_jp}\right)^{\frac{1}{p}}\cdot 2^{-2|\bm l|_{1}}\\
&=C_{p,d}\sum_{i=n+d}^{\infty} 2^{-2i}\cdot\sum_{|\bm l|_1=i}\left(\sum_{j=1}^d2^{l_jp}\right)^{\frac{1}{p}}\\
&\leq C_{p,d}\sum_{i=n+d}^{\infty} 2^{-2i}\cdot\sum_{|\bm l|_1=i}\left(\sum_{j=1}^d2^{2l_j}\right)^{\frac{1}{2}}\\
&\leq C_{p,d}\sum_{i=n+d}^{\infty} 2^{-i}\leq C_{p,d}2^{-n}\leq C_{p,d}W^{-2}L^{-2}.
\end{aligned}
\end{equation}
Combining this result with \eqref{Eq:sparse_error} and \eqref{Eq:psi_k_f_n}, we obtain
\begin{align*}
\|\psi_{\bm k}-f\|_{W^1_p(\Omega_{\bm k})}&\leq C_{p,d}\cdot  W^{-2}L^{-2}(\log( 2WL))^d,\\
\|\psi_{\bm k}-f\|_{L_{p}(\Omega_{\bm k})}&\leq C_{p,d}\cdot  W^{-4}L^{-4}(\log( 2WL))^{3d}.    
\end{align*}

\subsubsection{The case \texorpdfstring{$m\geq 3$}{m >= 3}}\label{ReluW1pcase2}
We need to first construct the desired function $\psi_{\bm k}$ to efficiently approximate the interpolation function
$$
\Pi_n^mf(\bm x)=\sum_{|\bm{l}|_1\leq n+d-1} \sum_{\bm{i} \in\mathcal{I}_{\bm l}} v_{\bm{l}, \bm{i}}^m \phi^m_{\bm{l}, \bm{i}}(\bm x).
$$
Combining this result with the error bound of $f$ and 
$\Pi_n^mf$ in the $W^1_p(\Omega)$  norm will yield to the desired results.

\textbf{Step 1} Construct the desired function $\psi_{\bm k}$ to { efficiently approximate} the interpolation function $\Pi_n^mf$.

To this end, we first construct a neural network $\psi_{\bm k}^{\bm l}$ to effectively approximate
$$
f_{\bm l}^m(\bm x)=\sum_{\bm{i} \in\mathcal{I}_{\bm l}} v_{\bm{l}, \bm{i}}^m \phi^{m}_{\bm{l}, \bm{i}}(\bm x).
$$
It has been demonstrated in the proof of Theorem \ref{Th:mainresult} that there exists a function $\phi_1\in \mathcal{NN}(9mdW,14L)$ such that for any $\bm x\in \Omega_{\bm i,\varepsilon}^{\bm l}$,
$$
\begin{aligned}
\phi_1(\bm x)&=\left(\psi_{l_1,i_1,1}(x_1),\ldots,\psi_{l_1,i_1,m-1}(x_1),\ldots,\psi_{l_d,i_d,1}(x_d),\ldots,\psi_{l_d,i_d,m-1}(x_d),\bm x^{\top}\right)^{\top},
\end{aligned}
$$
where $0\leq \psi_{l_j,i_j,k}=2^{-n-d+1}\rho_{l_j,i_j,k}\leq 1$ for any $j=1,2,\ldots,d$ and $k=1,2,\ldots,m-1$. By Lemma 3.5 in \cite{HonYang2022simultaneous} with $k=m-1$, there exists a function $\phi_2 \in (9(W+1)+m-2,14 (m-1)(m-2)d L)$ such that $\|\phi_2\|_{W^1_{ \infty}\left([0,1]^{m-1}\right)} \leq 18$ and
$$
\left\|\phi_2(\bm{x})-x_1 x_2 \cdots x_{m-1}\right\|_{W^1_{ \infty}\left([0,1]^{m-1}\right)} \leq 10m(W+1)^{-7 (m-1)d L}.
$$ 
Similar to the proof of the case $m\geq 3$ in Theorem \ref{Th:mainresult}, we obtain a function $$\widetilde{\phi}\in \mathcal{NN}(20mdW,28m^2d^2L)$$ such that
for any $\bm x\in \Omega_{\bm i,\varepsilon}^{\bm l}$ we have
$$
\widetilde{\phi}(\bm x)=\left(\begin{array}{c}
\widetilde{\psi}_{l_1,i_1}(x_1)\\
\widetilde{\psi}_{l_2,i_2}(x_2)\\

\vdots\\
\widetilde{\psi}_{l_d,i_d}(x_d)\\
\bm x
\end{array}\right),
$$
and it satisfies
\begin{equation}\label{Eq:widetilde_psi_error_w1p}
\begin{aligned}
\left\|\widetilde{\psi}_{l_j,i_j}(x_j)-\prod_{k=1}^{m-1}\psi_{l_j,i_j,k}(x_j)\right\|_{W^1_{ \infty}\left(\Omega_{i_j,\varepsilon}^{l_j}\right)}\leq 10m(W+1)^{-7 (m-1)d L},\,\,
|\widetilde{\psi}_{l_j,i_j}(x_j)|\leq 2
\end{aligned}
\end{equation}
for all $j=1,2,\ldots,d$. For $g_2\in \mathcal{NN}(4dW+5d,8L+5)$ defined in \eqref{Eq:g_2}, let $\phi_3=g_2\circ\widetilde{\phi}$. Then $\phi_3\in \mathcal{NN}(20mdW,41m^2d^2L)$ and
for any $\bm x\in \Omega_{\bm i,\varepsilon}^{\bm l}$ we have
$$
\phi_3(\bm x)=\left(\begin{array}{c}
\widetilde{\psi}_{l_1,i_1}(x_1)\\
\widetilde{\psi}_{l_2,i_2}(x_2)\\

\vdots\\
\widetilde{\psi}_{l_d,i_d}(x_d)\\
\operatorname{ind}(\bm i)
\end{array}\right).
$$
Since 
$|v^m_{\bm l,\bm i}|\leq C_{m,d}$ for any $\bm i\in \mathcal{I}_{\bm l}$ and $$\max_{\bm i\in \mathcal{I}_{\bm l}}|\operatorname{ind}(\bm i)|\leq \prod_{j=1}^d2^{l_j-1}-1\leq W^2(2L)^2-1,$$ by Lemma \ref{Le:fitsample}, there exists a function $$\phi_4\in \mathcal{NN}(16 s(W+1) \log(8 W),5(2L+2)\log(8 L))$$
such that $0 \leq \phi_4(x) \leq 1$ for any $x \in \mathbb{R}$,
satisfying
$$
\Big|\phi_4(\operatorname{ind}(\bm i))-\frac{v^m_{\bm l,\bm i}+C_{m,d}}{2C_{m,d}}\Big|\leq W^{-2s}(2L)^{-2s}
$$
for any $\bm i\in \mathcal{I}_{\bm l}$. Consequently, let $\widetilde{\phi}_4:=2C_{m,d}\phi_4-C_{m,d}\in \mathcal{NN}(16 s(W+1) \log(8 W),5(2L+2)\log(8 L))$ with the uniform upper bound $C_{m,d}$ in $\mathbb R$ and $g_6$ be the function satisfying 
$$g_6\binom{\bm z_1}{z_2}=\left(\sigma(\bm z_1)-\sigma(-\bm z_1),\widetilde{\phi}_4( z_2)\right)^{\top},
$$
for any $\bm z_1\in \mathbb R^d, z_2\in \mathbb R$. Using Lemma \ref{Le:concate}, we have
\begin{align*}
g_6&\in \mathcal{NN}(16 s(W+1) \log(8 W)+2d,5(2L+2)\log(8 L)) ,\\
\phi_5&:= g_6 \circ \phi_3
\in \mathcal{NN}(34smdW \log(8 W),61m^2d^2L\log(8 L))  
\end{align*}
such that for any $\bm x\in \Omega_{\bm i,\varepsilon}^{\bm l}$, it holds that
$$
\phi_5(\bm x)=\left(\begin{array}{c}
\widetilde{\psi}_{l_1,i_1}(x_1)\\
\widetilde{\psi}_{l_2,i_2}(x_2)\\

\vdots\\
\widetilde{\psi}_{l_d,i_d}(x_d)\\
\widetilde{v}^m_{\bm l,\bm i}
\end{array}\right),
$$
where $\widetilde{v}^m_{\bm l,\bm i}=\widetilde{\phi}_4(\operatorname{ind}(\bm i))$ and \begin{equation}\label{Eq:v_bound_W1p}
|\widetilde{v}^m_{\bm l,\bm i}-v^m_{\bm l,\bm i}|\leq 2C_{m,d}W^{-2s}(2L)^{-2s},\quad \text{for any}\,\, \bm i\in \mathcal{I}_{\bm l}.
\end{equation}
By utilizing Lemma \ref{Le:polynomial_W1p} with $c=3C_{m,d}\geq 3m$, we obtain a function $\phi_6\in (9(W+1)+2d+1, 28m d^2 L)$ such that $\|\phi_6\|_{W^1_{ \infty}\left(\Omega_c\right)} \leq C_{m,d}$ and
$$
\left\|\phi_6(\bm{x})-x_1 x_2 \cdots x_{d+1}\right\|_{W^1_{ \infty}\left(\Omega_c\right)} \leq C_{m,d}(W+1)^{-14d mL}.
$$    
Let $\phi^{\bm l}_{\bm k}:=2^{md(n+d-1)}\cdot\phi_6\circ \phi_5$. Then, by Lemma \ref{Le:concate},
$$\phi^{\bm l}_{\bm k}\in \mathcal{NN}(34smdW\log (8W),89m^2d^2L\log (8L))$$ and we have
$$
\begin{aligned}
&\|\phi^{\bm l}_{\bm k}(\bm x)-f^m_{\bm l}(\bm x)\|_{W^1_{\infty}(\Omega_{\bm k})}\\
\leq &\|\phi^{\bm l}_{\bm k}(\bm x)-f^m_{\bm l}(\bm x)\|_{W^1_{\infty}(\Omega_{\varepsilon}^{\bm l})}\\
\leq &\max_{\bm i\in \mathcal{I}_{\bm l}}\|\phi^{\bm l}_{\bm k}(\bm x)-v^m_{\bm{l}, \bm{i}} \phi^{m}_{\bm{l}, \bm{i}}(\bm x)\|_{W^1_{\infty}(\Omega_{\bm i,\varepsilon}^{\bm l})}\\
=&2^{md(n+d-1)}\cdot\max_{\bm i\in \mathcal{I}_{\bm l}}\Big\|\phi_6\circ \phi_5(\bm x)-v^m_{\bm{l}, \bm{i}} \prod_{j=1}^d \prod_{k=1}^{m-1} \psi_{l_j,i_j,k}(x_j)\Big\|_{W^1_{\infty}(\Omega_{\bm i,\varepsilon}^{\bm l})}\\
\leq &2^{md(n+d-1)}\cdot\left(\max_{\bm i\in \mathcal{I}_{\bm l}}\Big\|\phi_6\circ \phi_5(\bm x)-\widetilde{v}^m_{\bm{l}, \bm{i}} \prod_{j=1}^d \widetilde{\psi}_{l_j,i_j}(x_j)\Big\|_{W^1_{\infty}(\Omega_{\bm i,\varepsilon}^{\bm l})}\right.\\
&+ \max_{\bm i\in \mathcal{I}_{\bm l}}\Big\|\widetilde{v}^m_{\bm{l}, \bm{i}} \prod_{j=1}^d \widetilde{\psi}_{l_j,i_j}(x_j)-v^m_{\bm{l}, \bm{i}} \prod_{j=1}^d \widetilde{\psi}_{l_j,i_j}(x_j)\Big\|_{W^1_{\infty}(\Omega_{\bm i,\varepsilon}^{\bm l})}\\
&+\left.\max_{\bm i\in \mathcal{I}_{\bm l}}\Big\|v^m_{\bm{l}, \bm{i}} \prod_{j=1}^d \widetilde{\psi}_{l_j,i_j}(x_j)-v^m_{\bm{l}, \bm{i}} \prod_{j=1}^d \prod_{k=1}^{m-1} \psi_{l_j,i_j,k}(x_j)\Big\|_{W^1_{\infty}(\Omega_{\bm i,\varepsilon}^{\bm l})}\right)\\
:=&E_1+E_2+E_3.
\end{aligned}
$$
As for $E_1$, we have
$$
\begin{aligned}
E_1&\leq C_{m,d}2^{md(n+d-1)}\cdot
\max_{\bm i\in \mathcal{I}_{\bm l}}\max_{1\leq j\leq d}|\widetilde{\psi}_{l_j,i_j}(x_j)|_{W^1_{\infty}(\Omega_{i_j,\varepsilon}^{l_j})}\cdot(W+1)^{-14  mdL}\\
&\leq C_{m,d}W^{2md+2}L^{2md+2}(W+1)^{-14mdL}\\
&\leq C_{m,d}W^{-2m}L^{-2m}.
\end{aligned}
$$
As for $E_2$, we have
$$
\begin{aligned}
E_2&\leq \max_{\bm i\in \mathcal{I}_{\bm l}} |\widetilde{v}^m_{\bm{l}, \bm{i}}-v^m_{\bm{l}, \bm{i}}|\cdot\Big\| \prod_{j=1}^d \widetilde{\psi}_{l_j,i_j}(x_j)\Big\|_{W^1_{\infty}(\Omega_{\bm i,\varepsilon}^{\bm l})}\\
&\leq 2C_{m,d}W^{-2s}(2L)^{-2s}\cdot\max_{\bm i\in \mathcal{I}_{\bm l}} \left(\Big\| \prod_{j=1}^d \widetilde{\psi}_{l_j,i_j}(x_j)\Big\|_{L_{\infty}(\Omega_{\bm i,\varepsilon}^{\bm l})},\right.\\
&\quad\left.\max_{1\leq k\leq d}\Big\|\widetilde{\psi}_{l_k,i_k}(x_k)^{\prime} \prod_{j\neq k}^d \widetilde{\psi}_{l_j,i_j}(x_j)\Big\|_{L_{\infty}(\Omega_{\bm i,\varepsilon}^{\bm l})}\right)\\
&\leq C_{m,d}W^{-2s+2}L^{-2s+2}\\
&\leq C_{m,d}W^{-2m}L^{-2m},
\end{aligned}
$$
by taking $s=m+1$ and using inverse estimates and Equation \eqref{Eq:widetilde_psi_error_w1p} for the second inequality.
As for $E_3$, by Lemma \ref{Le:prod_approx_difference} and \eqref{Eq:widetilde_psi_error_w1p}, we have
$$
\begin{aligned}
E_3&\leq C_{m,d}
\max_{\bm i\in \mathcal{I}_{\bm l}}\Big\| \prod_{j=1}^d \widetilde{\psi}_{l_j,i_j}(x_j)- \prod_{j=1}^d \prod_{k=1}^{m-1} \psi_{l_j,i_j,k}(x_j)\Big\|_{W^1_{\infty}(\Omega_{\bm i,\varepsilon}^{\bm l})}\\
&\leq C_{m,d}(W+1)^{-7(m-1)dL}\\
&\leq C_{m,d}W^{-2m}L^{-2m}.
\end{aligned}
$$
Thus,
$\phi^{\bm l}_{\bm k}\in \mathcal{NN}(68m^2dW\log (8W),89m^2d^2L\log (8L))$,
$$
\|\phi^{\bm l}_{\bm k}(\bm x)-f^m_{\bm l}(\bm x)\|_{W^1_{\infty}(\Omega_{\bm k})}\leq C_{m,d} W^{-2m}L^{-2m}.
$$
Define $\psi_{\bm k}=\sum_{|\bm{l}|_1\leq n+d-1}\phi^{\bm l}_{\bm k}$. By Lemma \ref{Le:size_transition}, $$\psi_{\bm k}\in \mathcal{NN}(C_4W(\log(8W))^{d+1},C_5L(\log(8L))^{d+1})$$ and
we have 
\begin{equation}\label{Eq:psi_k_I_nf}
\begin{aligned}
\|\psi_{\bm k}-\Pi_n^mf\|_{W^1_{\infty}(\Omega_{\bm k})} &\leq \sum_{|\bm{l}|_1\leq n+d-1}\|\phi^{\bm l}_{\bm k}(\bm x)-f_{\bm l}^m(\bm x)\|_{W^1_{\infty}(\Omega_{\bm k})}\\
&\leq C_{m,d}\cdot  W^{-2m}L^{-2m}(\log( 2WL))^d,
\end{aligned}
\end{equation}
where $C_4=68 d(2d)^dm^2$ and $C_5=89d^2 m^2$.

\textbf{Step 2} Calculate the error bound of $f$ and 
$\Pi_n^mf$ in the $W^1_p(\Omega)$  norm.

We first note that
$$
\|f-\Pi_n^mf\|_{W^1_p(\Omega)}\leq\sum_{|\bm{l}|_1> n+d-1} \Big\|\sum_{\bm{i} \in\mathcal{I}_{\bm l}} v_{\bm{l}, \bm{i}}^m \phi^m_{\bm{l}, \bm{i}}\Big\|_{W^1_p(\Omega)}.
$$
For any $\bm l$ satisfying $|\bm{l}|_1> n+d-1$, by the results 
$$
\|\phi_{\bm{l}, \bm{i}}^m\|_{L_p(\Omega)} \leq 1.117^d \cdot 2^{d / p} \cdot 2^{-|\bm{l}|_1 / p}, \quad p\in [1,\infty],
$$ in \cite{Bungartz_Griebel2004} and the facts (see \cite[Lemma 4.6]{Bungartz_Griebel2004} and \cite{lizhang2025korobov})
$$
|v_{\bm l,\bm i}^m|\leq  C_{m,d} \cdot 2^{-(m-1/p)|\bm l|_1}\cdot\left\|D^{\bm{m}} f\right\|_{L_p(\Omega_{\bm l,\bm i})}, 
$$
we have
$$
\begin{aligned}
\Big\|\sum_{\bm{i} \in\mathcal{I}_{\bm l}} v_{\bm{l}, \bm{i}}^m \phi^m_{\bm{l}, \bm{i}}(\bm x)\Big\|_{W^1_p(\Omega)}^p 
=&\sum_{\bm{i} \in\mathcal{I}_{\bm l}} |v_{\bm{l}, \bm{i}}^m|^p\| \phi^m_{\bm{l}, \bm{i}}(\bm x)\|_{W^1_p(\Omega_{\bm l,\bm i})}^p\\
\leq &C_{d,p }\left(1+\sum_{j=1}^d2^{pl_j}\right) \sum_{\bm{i} \in\mathcal{I}_{\bm l}} |v_{\bm{l}, \bm{i}}^m|^p\|\phi^m_{\bm l,\bm i}\|_{L_p(\Omega_{\bm l,\bm i})}^p\\
\leq &C_{m,d,p }\left(1+\sum_{j=1}^d2^{pl_j}\right)\cdot 2^{-mp|\bm l|_1}
,
\end{aligned}
$$
where the first inequality is due to the well-known inverse estimates (see \cite[Lemma 4.5.3]{Brenner_Scott1994} for instance). Thus, similar to \eqref{Eq:f_fn_W1p_error}, we have
$$
\begin{aligned}
\|f-\Pi_n^mf\|_{W^1_p(\Omega)}&\leq C_{m,d,p } \sum_{|\bm{l}|_1> n+d-1} \left(\sum_{j=1}^d2^{pl_j}\right)^{1/p}\cdot 2^{-m|\bm l|_1}\\
&\leq C_{m,d,p } \sum_{i=n+d}^{\infty}2^{-mi}\cdot\sum_{|\bm{l}|=i} \left(\sum_{j=1}^d2^{pl_j}\right)^{1/p}\\
&\leq C_{m,d,p } \sum_{i=n+d}^{\infty}2^{-(m-1)i}\\
&\leq C_{m,d,p}2^{-(m-1)n}\leq C_{m,d,p }W^{-2(m-1)}L^{-2(m-1)}.
\end{aligned}
$$
Therefore, combining this result with \eqref{Eq:Sparse_error_If_n}, \eqref{Eq:psi_k_I_nf}, we have
\begin{align*}
\|\psi_{\bm k}-f\|_{W^1_p(\Omega_{\bm k})}&\leq C_{p,d}\cdot  W^{-2(m-1)}L^{-2(m-1)}(\log( 2WL))^d,\\
\|\psi_{\bm k}-f\|_{L_{p}(\Omega_{\bm k})}&\leq C_{p,d}\cdot  W^{-2m}L^{-2m}(\log( 2WL))^{3d}.
\end{align*}
The proof is complete.

\vspace{2mm}
\begin{remark}
For $f\in X^m_{\infty}(\Omega)$, it is worth noting that the error of approximating sparse grid interpolation function using ReLU DNNs in the proof of Proposition \ref{Le:patch_psik_error} and the used lemmas are all valid for the norm $W^1_{\infty}$. Therefore, once the sparse grid interpolation error has been established in the $W^1_{\infty}$ norm, the corresponding approximation error for ReLU DNNs applied to $f \in X^m_{\infty}(\Omega)$ can be derived under the $W^1_{\infty}$ norm.    
\end{remark}
\vskip2mm

\section{Conclusion and Future Work}\label{Sec:conclusion}
Using the bit-extraction technique, we have proven nearly optimal super-approximation $L_p$ and $W^1_p$ error bounds of ReLU DNNs for Korobov functions ($1\leq p<\infty$). Furthermore, we also provide remarks on the approximation performance under certain other network structures and norms.
The bit extraction techniques based on ReLU networks, Floor-ReLU networks (\cite{shen2021deepfloor}), or FLES networks (\cite{shenyangzhang2021neuralthree}) have been well developed. Nevertheless, investigating bit extraction techniques based on other activation functions will lead to a range of network approximation characteristics across various function spaces. This will remain a future research direction. 

%To the best of our knowledge, research that combines bit extraction with DNNs are all established on uniform grids, such as Lemma \ref{Le:stepfunction}.
% Moreover, the proof techniques used in Lemma \ref{Le:stepfunction} 
%also leverage the uniformity of the grid.  
%Moreover, the functions being approximated in this paper are presumed to have high regularity. %However, in the numerical treatment of PDEs, the solutions are not always that regular, especially in cases like corner singularities in elliptic problems. The sparse tensor product approximation used in this paper is severely limited in such scenarios. To address this situation, 
%The nonuniform distribution of degrees of freedom through local mesh refinement has been proposed to approximate functions in certain weighted spaces where the regularity of functions are not as that of this paper. 
%In \cite{Nitsche2004sparse}, the author utilized anisotropically refined sparse grid interpolation to provide approximation rates for functions lacking the high regularity assumed in this paper. %Investigating whether ReLU DNNs can achieve a super-approximation rate for such functions will involve the consideration of non-uniform grids. Exploring the approximation of DNNs combined with bit extraction on non-uniform grids to potentially provide a super-approximation rate for such functions would be very interesting and remains a future research direction.
\section*{Acknowledgments}
Y. L. was supported by the National Natural Science Foundation of China under grant 12471346. The authors would like to thank the anonymous referees for helpful remarks that improved
the quality and presentation of this paper.

\bibliographystyle{plain}
\bibliography{r}
%%%%%%%%%%%%%%%%%%%%%%%%%%%%%%%%%%%
\end{document}